\documentclass[12pt]{article}

\usepackage{cite} 
\usepackage{url}  
\usepackage{ifthen}  
\usepackage{multicol}   
\usepackage{multirow}
\usepackage{hhline}
\usepackage[utf8]{inputenc} 
\usepackage{fixltx2e} 
\usepackage{fullpage}
\usepackage{graphicx}
\usepackage{float}
\usepackage{color}
\usepackage{amsmath}
\usepackage{amsthm}
\usepackage{amssymb}
\usepackage{multirow}
\usepackage{rotating}
\usepackage[top=1.2in, bottom=1.5in, left=1in, right=1in]{geometry}

\usepackage[square,numbers]{natbib}
\usepackage{epsfig}
\usepackage{booktabs}

\usepackage{hyperref} 

\usepackage[usenames,dvipsnames,svgnames,table]{xcolor}

\usepackage{enumitem}

\usepackage{algorithm}
\usepackage{algpseudocode}

\usepackage{mathtools}

\usepackage[nodisplayskipstretch]{setspace}

\allowdisplaybreaks

\newtheorem{theorem}{Theorem}
\newtheorem{lemma}{Lemma}
\newtheorem{corollary}{Corollary}
\newtheorem{proposition}{Proposition}

\newcommand{\argmin}[1]{\operatorname*{argmin}_{#1}}
\newcommand{\argmax}[1]{\operatorname*{argmax}_{#1}}

\newcommand{\beginsupplement}{%
        \setcounter{table}{0}
        \renewcommand{\thetable}{S\arabic{table}}%
        \setcounter{figure}{0}
        \renewcommand{\thefigure}{S\arabic{figure}}%
        \setcounter{section}{0}
        \renewcommand{\thesection}{S\arabic{section}}%
        \setcounter{equation}{0}
        \renewcommand{\theequation}{S\arabic{equation}}%
     }
     
\makeatletter
\let\@fnsymbol\@arabic
\makeatother

\newcommand{\calX}{\mathcal{X}}
\newcommand{\calY}{\mathcal{Y}}
\newcommand{\calP}{\mathcal{P}}
\newcommand{\bbR}{\mathbb{R}}

\newcommand{\bw}{\boldsymbol{w}}
\newcommand{\bx}{\boldsymbol{x}}
\newcommand{\bX}{\boldsymbol{X}}
\newcommand{\boldpi}{\boldsymbol{\pi}}

\newcommand{\bolddelta}{\boldsymbol{\delta}}

\newcommand{\hatY}{\widehat{Y}}

\newcommand{\intervals}{\Omega_{\boldpi}}

\newcommand{\bayesint}{W_{\boldpi}^*}

\newcommand{\surrplus}{\phi^+}
\newcommand{\surrminus}{\phi^-}
\newcommand{\surrogate}{\phi}
\newcommand{\bayessurr}{f^*_{\phi}}

\newcommand{\varsurrplus}{\varphi^+}
\newcommand{\varsurrminus}{\varphi^-}
\newcommand{\varsurrogate}{\varphi}

\newcommand{\ebpi}{\ell_{\boldsymbol{\pi}}}

\title{\vskip-.5in \mbox{Large-Margin Classification with Multiple Decision Rules}}
\author{%
Patrick~K.~Kimes$^1$ \and %
D.~Neil~Hayes$^2$ \and %
J.~S.~Marron$^{1,3}$ \and %
Yufeng~Liu$^{1,3,4,\dagger}$ \and 
for the Alzheimer's Disease Neuroimaging Initiative$^*$%
}

\date{}

\begin{document}

\maketitle

\noindent
{\bf Abstract:}
Binary classification is a common statistical learning problem in which a model is estimated on a set of covariates for some outcome indicating the membership of one of two classes. In the literature, there exists a distinction between hard and soft classification. In soft classification, the conditional class probability is modeled as a function of the covariates. In contrast, hard classification methods only target the optimal prediction boundary. While hard and soft classification methods have been studied extensively, not much work has been done to compare the actual tasks of hard and soft classification. In this paper we propose a spectrum of statistical learning problems which span the hard and soft classification tasks based on fitting multiple decision rules to the data. By doing so, we reveal a novel collection of learning tasks of increasing complexity. We study the problems using the framework of large-margin classifiers and a class of piecewise linear convex surrogates, for which we derive statistical properties and a corresponding sub-gradient descent algorithm. We conclude by applying our approach to simulation settings and a magnetic resonance imaging (MRI) dataset from the Alzheimer's Disease Neuroimaging Initiative (ADNI) study.
\vskip.15in


\let\thefootnote\relax\footnote{\noindent
$^1$Department of Statistics and Operations Research, 
$^2$Lineberger Comprehensive Cancer Center, 
$^3$Department of Biostatistics,
$^4$Carolina Center for Genome Sciences, University of North Carolina, Chapel Hill, NC 27599.
$^\dagger$Correspondence to: Yufeng Liu (yfliu@email.unc.edu).
$^*$Data used in preparation of this article were obtained from the Alzheimer's Disease Neuroimaging Initiative (ADNI) database (adni.loni.usc.edu). As such, the investigators within the ADNI contributed to the design and implementation of ADNI and/or provided data but did not participate in analysis or writing of this report. A complete listing of ADNI investigators can be found at: \url{http://adni.loni.usc.edu/wp-content/uploads/how_to_apply/ADNI_Acknowledgement_List.pdf.} 
}

\thispagestyle{empty}


\clearpage
\setcounter{page}{1}

\section{Introduction}
\label{intro}

Classification is one of the most widely applied and well studied problems in supervised learning. Given a training set of observed covariates and outcomes, similar to the usual regression problem, in classification, the outcome is modeled as a function of the set of covariates. However, in contrast to standard regression with a continuous response variable, classification describes the setting where the outcome is a discrete class label. While generalizations to more than two classes exist, in this paper we focus on the standard binary problem where the label takes one of two possible values, typically denoted by $+1$ and $-1$.

Given such a dataset, commonly, the goal is to build a model, either to predict the class of a new observation from the covariate space, or to estimate the probably of each class as a function of the covariates. The tasks correspond respectively to hard and soft classification. Briefly, we refer to methods which only target the optimal prediction rule as hard classifiers, and those which produce estimates of class probability as soft classifiers. Examples of hard classifiers include the support vector machine (SVM) \cite{Vapnik1995,Vapnik1998} and $\psi$-learning \cite{Shen2003,Liu2006}, and examples of soft classifiers include logistic regression and other likelihood-based approaches. Often, soft classifiers are also used to obtain hard classification rules by predicting the class with greater estimated probability. These rules are commonly referred to as plug-in classifiers. While hard classification rules do not directly provide conditional class probability estimates, several approaches have been proposed for estimating class probabilities based on hard classifiers, including those of \cite{Platt1999} and \cite{Wang2008}. As such, methods which may be traditionally viewed as soft and hard classifiers are often used for either task. Naturally, a question of interest is: how are hard and soft classifiers related, and how do they differ in practice?

Recently, \cite{Liu2011a} introduced the Large-margin Unified Machines (LUM) family of margin-based classifiers, shedding some light on the the relationship between hard and soft classifiers. The LUM family connects several popular margin-based classification methods, including SVM, distance-weighted discrimination (DWD) \cite{Marron2007}, and a new hybrid logistic loss. Their approach was further extended to the multi-category case in \cite{Zhang2013}. Margin-based approaches to classification are popular in practice for their accuracy and computational efficiency in both low and high-dimensional settings. While a flexible family of margin-based classifiers, the LUM approach examines only a specific parameterized collection of classifiers along the gradient of soft to hard classification. In this paper, we similarly focus on connecting hard and soft margin-based methods. However, we consider a more natural approach based on connecting the tasks of hard and soft classification rather than specific hard and soft classifiers. Specifically, we propose a novel framework of binary learning problems which may be formulated as partial or full estimation of the conditional class probability based on fitting an arbitrary number of boundaries to the data. As an example, suppose we are interested in separating patients into four disease risk groups based on clinical measurements. One possible approach is to group patients according to whether their conditional probability of being positive for the disease is less than 25\%, between 25\% to 50\%, between 50\% to 75\%, or greater than 75\%. In this setting, the emphasis is not on the accuracy of class probability estimates, but instead, on the correct stratification of individuals into risk groups. Therefore, only partial estimation of the conditional class probability is required; in particular, at the three boundaries, 25\%, 50\%, and 75\%. While stratification of the patient classes is possible using a soft classifier, an approach directly targeting the three boundaries may provide improved stratification by requiring weaker assumptions on the entire form of the underlying conditional class probability. 

In addition to hard and soft classification, the proposed framework also encompasses rejection-option classification \cite{Herbei2006, Bartlett2008,Yuan2010a,Wegkamp2011} and weighted classification \cite{Lin2002a,Qiao2009}, two other well-studied binary learning problems. Briefly, the rejection-option problem expands on standard binary classification by introducing a third option to reject, where neither label is predicted. Notably, it can be shown that the decision to reject directly corresponds to a prediction that the probability of belonging to either class does not exceed a specified threshold. Since the task requires estimation of more than a single classification boundary, but less than the full class conditional probability, it may be viewed as an intermediate problem to hard and soft classification, as in the example given above. Applications of rejection option classification include certain medical settings where predictions should only be made when a level of certainty is obtained. Additionally, weighted classification extends the standard classification problem by accounting for differences or biases in class populations. We define these problems more formally, along with hard and soft classification, in Section~\ref{framework}.

The remainder of this paper is organized as follows. In the first part of Section~\ref{framework} we provide a review of margin-based learning. Then, in the remainder of Section~\ref{framework}, we define our family of binary learning problems and introduce a corresponding \emph{theoretical loss}, which generalizes the standard misclassification error to connect class prediction with probability estimation. In Section~\ref{losses} we provide necessary and sufficient conditions for consistency of a surrogate loss function, and propose a class of consistent piecewise linear surrogates akin to the SVM hinge loss for binary classification. In Section~\ref{statprops}, we present theoretical bounds on the empirical performance of classification rules obtained using surrogate loss functions. In Section~\ref{computation}, we provide a sub-gradient descent (SGD) algorithm for solving the corresponding optimization problem using the proposed piecewise linear surrogates. We then illustrate the behavior of our generalized family of classifiers using simulation in Section~\ref{simulations}, and a real data example from the Alzheimer's Disease Neuroimaging Initiative (ADNI) database in Section~\ref{realdata}. We conclude in Section~\ref{discussion} with a discussion of the proposed framework.

\section{Methodology}
\label{framework}

In this section, we first briefly introduce margin-based classifiers, and formally define the notion of classification consistency for loss functions. We then state the general form of our unified framework of problems and introduce a corresponding family of theoretical loss functions which encompasses the standard misclassification error as a special case.

\subsection{Margin-Based Classifiers}
\label{framework:margin}
Let $\{(\bx_i,y_i)\}_{i=1}^n$ denote a training set of $n$ covariate--label pairs drawn from $\calX \times \calY$ according to some unknown distribution $\calP(\bX,Y)$. For binary problems, $\calY = \{-1,+1\}$ is used to denote the label space, and often $\calX \subset \mathbb{R}^p$, with $p\geq1$. Given a training set, margin-based classifiers minimize a penalized loss over a class, $\mathcal{F}$, of margin functions, $f:\calX \rightarrow \mathbb{R}$. Typically, the corresponding optimization problem is written as:
\begin{align}
	\min_{f\in\mathcal{F}}\ \frac{1}{n}\sum_{i=1}^n \underbrace{L\big(y_if(\bx_i)\big)}_{loss} \ 
		+ \ \underbrace{\lambda J(f)}_{penalty}, \label{eq:loss+penaltyCh}
\end{align}
where $L:\mathbb{R}\rightarrow\mathbb{R}$ is a loss function defined with respect to the functional margin, $yf(\bx)$, and $J:\mathcal{F}\rightarrow\mathbb{R}$ is some roughness measure on $\mathcal{F}$ with corresponding tuning parameter $\lambda \geq 0$. Both hard and soft classification may be formulated as margin-based problems. In the case of hard classification, with a little abuse of notation, we use $\hatY\in\calY$ to denote a predicted class label, and $\hatY:\mathbb{R}\rightarrow\calY$ to denote a prediction rule on $\mathbb{R}$. In margin-based classification, $\hatY(\cdot)$ is combined with a margin function, $f\in\mathcal{F}$, to obtain predictions in $\calY$. Most commonly, in hard classification the sign rule, $\hatY(f(\bX)) = \text{sign}(f(\bx))$, is used, assuming $f(\bx) \not= 0$ almost surely (\textit{a.s.}). Thus, given a new $(\bx^*,y^*)$ pair with $f(\bx^*)\not=0$, correct classification is obtained if and only if $y^*f(\bx^*)>0$. Since the functional margin, $yf(x)$, serves as an approximate measure for classification correctness, the loss function, $L$, in~(\ref{eq:loss+penaltyCh}) is often chosen to be a non-increasing function over $yf(\bx)$. A natural choice of $L$ in hard classification is the misclassification error, or 0$-$1 loss, given by:
\begin{align}
	\ell_{0-1}(Y,\hatY) = \textbf{I}\{\hatY \not= Y\}, \label{eq:01loss}
\end{align}
where $\textbf{I}\{\cdot\}$ is used to denote the indicator function. Using the sign rule, the loss may be equivalently written over the class of margin functions as: $L_{0-1}(Yf(\bX)) = \textbf{I}\{Yf(\bX)<0\}$. However, direct optimization of the non-convex and discontinuous loss, $L_{0-1}$, is NP-hard and often infeasible in practice. Thus, continuous convex losses, called {surrogates}, are commonly used instead. Choices of the surrogate loss function corresponding to existing margin-based classifiers include the SVM hinge loss, $L(z) = \max\{0,1-z\}$, logistic loss, $L(z) = \log(1+e^{-z})$, and the DWD loss, $L(z) = \tfrac{1}{4z} \cdot \textbf{I}\{z \geq \tfrac{1}{2}\} + (1-z) \cdot \textbf{I}\{z < \tfrac{1}{2}\}$. Finally, the penalty term, $J(\cdot)$ is used to prevent over-fitting and improve generalizability of the resulting classifier. The amount of penalization is commonly determined by cross-validation over a grid of $\lambda$ values. Here, we note that while in the literature there exists a natural theoretical loss for hard classification, i.e. the 0$-$1 loss, there is no equivalent theoretical loss targeting consistent probability estimation for soft classification. In addition to providing a spectrum of theoretical loss functions covering soft and hard classifications at the two extremes, our proposed framework also naturally defines precisely such a theoretical loss for the soft classification problem (Figure~\ref{fig:steplossEg-01}C). 

In Section~\ref{intro}, we briefly discussed the learning tasks of rejection-option and weighted classification. As with hard and soft classification, these tasks may also be formulated as margin-based problems. We next describe how rejection-option classification may be formulated as a problem of the form~(\ref{eq:loss+penaltyCh}). Borrowing the notation of \cite{Yuan2010a}, we use $0$ to denote the rejection option such that a prediction, $\hatY_{rej}$, takes values in $\calY_{rej} = \{+1,0,-1\}$. Then, for some pre-specified \emph{rejection cost} $\pi\in(0,\tfrac{1}{2})$, they propose the following theoretical loss for rejection-option classification: 
\begin{align}
	\ell_{rej,\pi} (Y, \hatY_{rej}) =
		\begin{cases}
			1 &\text{if } \hatY_{rej} \not= Y,\ \hatY_{rej} \not= 0 \\
			\pi &\text{if } \hatY_{rej} = 0 \\
			0 &\text{otherwise}
		\end{cases}. \label{eq:rejloss}
\end{align}
To express the loss as a function over $Yf(\bX)$, \cite{Yuan2010a} propose the prediction rule $\hatY_{rej}(f(\bX);\delta) = \textbf{I}\{|Yf(\bX)| >\delta\}\cdot\text{sign}(Yf(\bX))$ for some appropriately chosen $\delta>0$. Then, $\ell_{rej,\pi}$ may be written as the following generalized 0$-$1 loss on $Yf(\bX)$: 
\begin{align*}
	L_{rej,\pi}(Yf(\bX); \delta) = (1-\pi)\textbf{I}\{Yf(\bX)\leq-\delta\} + \pi\textbf{I}\{Yf(\bX)<\delta\}.
\end{align*}

We finally consider the task of weighted classification. In contrast to the problems mentioned thus far, to fit the form of (\ref{eq:loss+penaltyCh}), weighted classification requires specifying separate theoretical loss functions for observations from the $+1$ and $-1$ classes, denoted by $\ell_{\text{w},\pi}^+$ and  $\ell_{\text{w},\pi}^-$. For simplicity, we use $\ell_{\text{w},\pi}^Y$ to denote the loss for both classes. Similar to hard classification, the task is to predict class labels in $\calY=\{+1, -1\}$. The loss function depends on a weight parameter, $\pi$, which accounts for imbalances between the two classes. Commonly, $\pi$ is constrained to the interval $(0,1)$ without loss of generality. Then, for fixed weight $\pi$, the weighted loss is given by:
\begin{align}
	\ell^Y_{\text{w},\pi} (Y, \hatY) 
		&= \textbf{I}\{Y=+1\} \cdot \ell^+_{w,\pi}(\hatY) 
			+ \textbf{I}\{Y=-1\} \cdot \ell^-_{w,\pi}(\hatY), \label{eq:wloss} \\
	\ell^+_{\text{w},\pi} (\hatY) &= (1-\pi) \cdot \textbf{I}\{\hatY \not= +1\}, \notag \\
	\ell^-_{\text{w},\pi} (\hatY) &= \pi \cdot \textbf{I}\{\hatY \not= -1\}. \notag
\end{align}
Note that the standard 0$-$1 loss corresponds to the special case of the weighted loss~(\ref{eq:wloss}) when equal weight is assigned to the two classes with $\pi=\tfrac{1}{2}$. Using the same prediction rule as for hard classification, $\hatY(f(\bx)) = \text{sign}(f(\bx))$, the loss over the functional margin may be written:
\begin{align*}
	L^Y_{\text{w},\pi} (Yf(\bX)) &= \textbf{I}\{Y=+1\} \cdot L^+_{\text{w},\pi}(Yf(\bX)) 
					+ \textbf{I}\{Y=-1\} \cdot L^-_{\text{w},\pi}(Yf(\bX)), \\
	L^+_{\text{w},\pi} (Yf(\bX)) &= (1-\pi) \cdot \textbf{I}\{Yf(\bX) < 0\}, \\
	L^-_{\text{w},\pi} (Yf(\bX)) &= \pi \cdot \textbf{I}\{Yf(\bX) < 0\}.
\end{align*}
As with the usual 0$-$1 loss, optimization of $L_{rej,\pi}$ and $L^Y_{\text{w}, \pi}$ is NP-hard, and in practice should be approximated using a convex surrogate loss. In the next section, we introduce the notion of consistency, an important statistical property of surrogate loss functions.

\subsection{Classification Consistency}
Much work has been done to study the statistical properties of classifiers of the $loss+penalty$ form given in (\ref{eq:loss+penaltyCh}) \cite{Steinwart2007,Blanchard2008,Bartlett2006,Cristianini2000}. Of these, {consistency} of loss functions is one of the most fundamental. In general, a loss function is called consistent for a margin-based learning problem if it recovers in expectation the optimal rule, often called the Bayes rule, to the theoretical loss function, e.g. $\ell_{0-1}$, $\ell_{rej,\pi}$ or $\ell^Y_{\text{w}, \pi}$. More formally, for a theoretical loss function, $\ell$, and a surrogate loss, $\phi$, let $Y_{\ell}^*(\bX) = \argmin{Y^*} \mathbb{E}_{Y|\bX} \{ \ell(Y, Y^*) \}$ and $f^*_\phi (\bX) = \argmin{f} \mathbb{E}_{Y|\bX} \{ \phi(Yf(\bX)) \}$ denote the Bayes rule and $\phi$-optimal margin function, respectively. Then, we call $\phi$ consistent for $\ell$ if $\hatY_{\ell}(f^*_\phi (\bX)) = Y_{\ell}^*(\bX)$, where $\hatY_{\ell}$ is the appropriate prediction rule, e.g. the sign function. Equivalently, using the margin-based formulation of the theoretical loss, $L$, and letting $f^*_L (\bX) = \argmin{f} \mathbb{E}_{Y|\bX} \{ L(Yf(\bX)) \}$ denote the $L$-optimal margin function, consistency may be expressed as $\hatY_{\ell}(f^*_\phi (\bX)) = \hatY_{\ell}(f^*_L (\bX))$. For rejection-option classification, the Bayes optimal rule is given by:
\begin{align}
	Y_{rej,\pi}^*(\bX) =
		\begin{cases}
			+1 &\text{if } p(\bX) \geq 1-\pi \\
			0 &\text{if } p(\bX) \in (\pi, 1-\pi) \\
			-1 &\text{if } p(\bX) \leq \pi
		\end{cases}. \label{eq:bayesrejection}
\end{align}
The Bayes optimal rule for weighted classification is given by:
\begin{align}
	Y_{\text{w},\pi}^*(\bX) =
		\begin{cases}
			+1 &\text{if } p(\bX) > \pi \\
			-1 &\text{if } p(\bX) \leq \pi
		\end{cases}. \label{eq:bayesweighted}
\end{align}
For hard classification, the Bayes optimal rule corresponds to $Y_{\text{w}, 0.5}$ , and consistency is often referred to as Fisher consistency or classification calibrated \cite{Bartlett2006}. While no theoretical loss has been proposed for soft classification, using $p(\bX)=\mathbb{P}(Y=+1 |\bX)$ to denote the conditional class probability at $\bX\in\calX$, commonly, $\phi$ is called consistent for soft classification if there exists some monotone mapping, $C:\mathbb{R}\rightarrow[0,1]$ such that $C(f^*_\phi(\bX)) = p(\bX)$. Naturally, $C(\cdot)$ may be viewed as an extension of the prediction rules $\hatY(\cdot)$ and $\hatY_{rej}(\cdot;\delta)$ given for hard and rejection-option classification. Necessary and sufficient conditions for Fisher, rejection-option, and probability estimation consistency have been described in \cite{Lin2002,Yuan2010a,Zhang2013a}. 

In this paper, we propose a novel framework for unifying hard, soft, rejection-option, and weighted classification through a generalized formulation of their corresponding theoretical losses, corresponding Bayes optimal rules, and necessary and sufficient conditions for consistency. Our generalized formulation not only provides a platform for comparing existing binary classification tasks, but also introduces an entire family of new tasks which fills the gap between these problems. We next formally introduce our unified framework of binary learning problems.

\subsection{Unified Framework}
\begin{figure}[t]
	\centering
	\includegraphics[width=\textwidth]{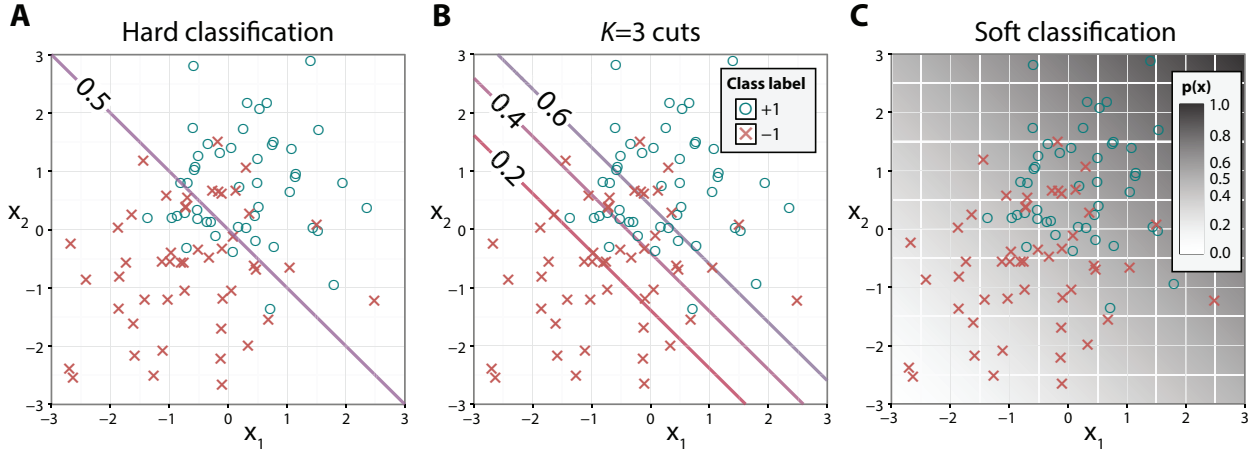}%
	\caption[Example of hard classification, set estimation and soft classification.]{Boundaries are shown separating the input space, $\mathbb{R}^2$ into the $K+1$ prediction sets, $\intervals$. A sample of 100 observations drawn from the underlying class populations are overlaid to show the distribution over the space. (A) The $K=1$ boundary for $\boldpi = \{0.5\}$ corresponding to hard classification is shown by the separating hyperplane corresponding to the set $\{\bx\in\mathbb{R}^2: p(\bx) = 0.5\}$. (B) The set of $K=3$ boundaries are shown for $\boldpi = \{0.2,0.4,0.6\}$ separating the 4 prediction sets. (C) The soft classification results are shown spanning the entire set of $\pi_k \in (0,1)$. As $K\rightarrow\infty$, moving from hard to soft classification, the set of learning problems becomes increasingly complex.}
	\label{fig:cutEg}
\end{figure}
First, we note that all of the classification tasks described in Section~\ref{framework:margin} may be formulated as learning problems which target partial or complete estimation of the conditional class probability, $p(\bx)$. We propose our framework of unified margin-based learning problems based on this insight. Let $\intervals$ denote the ordered $(K+1)$ partition of the interval $[0,1]$ obtained by splitting at $\boldpi = \{\pi_1,\ldots,\pi_K\}$, where $0 < \pi_1 < \ldots < \pi_{K} < 1$. Assume $p(\bx)\not=\pi_k$ \textit{a.s.} for all $k$, such that observations belong to only a single region of interval. Letting $\pi_0=0$ and $\pi_{K+1}=1$ for ease of notation, we write: 
\begin{align*}
	\intervals 
		&= \big\{ \omega_0, \ldots, \omega_K \big\},
\end{align*}
where $\omega_0 = [\pi_0,\pi_1]$, and $\omega_k = (\pi_{k}, \pi_{k+1}]$, for $k\geq1$. As our framework, we propose the class of problems which target a partition of the covariate space, $\calX$, into the $K+1$ regions, $\{\bx: p(\bx) \in \omega_k\}$. In Figure~\ref{fig:cutEg}, we show a sample of $100$ observations drawn from the same underlying distribution, $\mathcal{P}(\bX,Y)$ along with optimal solutions to three representative problems from our proposed framework. Note that the extreme cases of $K=1$ with $\boldpi=\{0.5\}$ (Figure~\ref{fig:cutEg}A), and $K=\infty$ with $\boldpi$ dense on $(0,1)$ (Figure~\ref{fig:cutEg}C) correspond to hard and soft classification, respectively. We discuss these connections in more detail later in this section. To illustrate the spectrum of problems in our framework, we also show a new intermediate problem in Figure~\ref{fig:cutEg}B, with $K=3$ and $\boldpi = \{0.2,0.4,0.6\}$.

Formally, we define our framework as the collection of minimization tasks of a theoretical loss which generalizes the 0$-$1 loss, over the collection of rules $\mathcal{G}_{\boldpi} = \{ g:\calX\rightarrow\intervals \}$. Recall the weighted 0$-$1 loss, $\ell_{\text{w}}$, for weighted classification described above. For positive and negative class weights $(1-\pi)$ and $\pi$ where $\pi\in(0,1)$, the weighted 0$-$1 loss has corresponding Bayes boundary at $\{\bx: p(\bx) = \pi\}$. Problems under our framework may be viewed as the task of simultaneously estimating $K$ such boundaries. Intuitively, we formulate our theoretical loss as the average of $K$ weighted 0$-$1 loss functions with corresponding weights $\boldpi$. Throughout, we use $\ell^+_{\boldpi}(g(\bx))$ and $\ell^-_{\boldpi}(g(\bx))$ to denote the loss for positive and negative class observations, respectively. As with the weighted loss, we use $\ell^Y_{\boldpi}$ to denote the loss for both classes:
\begin{align}
	\ell_{\boldpi}^Y (g(\bX)) &= 
			\frac{2}{K} \sum_{k=1}^K \ell^Y_{\pi_k} (g(\bX)), \label{eq:general01}
\end{align}
where 
\begin{align*}
	\ell_{\pi_k}^+ (g(\bX)) &= (1-\pi_k) \cdot \textbf{I}\{ g(\bX) \leq \pi_k \}, \\
	\ell_{\pi_k}^- (g(\bX)) &= \pi_k \cdot \textbf{I}\{ g(\bX) > \pi_k \},
\end{align*}
and the notion of inequalities is extended to elements of $\intervals$ such that $(\pi_{j},\pi_{j+1}] \leq \pi_k$ if $\pi_{j+1} \leq \pi_k$ and $(\pi_{j},\pi_{j+1}] > \pi_k$ if $\pi_{j} \geq \pi_k$. As we show in Supplementary Section~S1, our theoretical loss encompasses the usual 0$-$1 loss, its weighted variant, and the rejection-option loss proposed by \cite{Yuan2010a}. The multiplicative constant, 2, is included in~(\ref{eq:general01}) such that $\ell^Y_{\boldpi}$ reduces precisely to the usual 0$-$1 loss when $\boldpi = \{0.5\}$. Note that since $\ell^Y_{\boldpi}$ is effectively the average of $K$ indicator functions scaled by 2, the function takes values in the interval $[0,2]$. In Figure~\ref{fig:steplossEg-01}, we show $\ell^Y_{\boldpi}$ as a function of $g(\bx) \mapsto \intervals$, corresponding to the problems in Figure~\ref{fig:cutEg}. Along the horizontal axis, the range $[0,1]$ is split into corresponding $\omega_j =  (\pi_{j},\pi_{j+1}]$ intervals. Note that the loss function is constant within each interval, giving the appearance of a step function, except in the extreme case when $K=\infty$. As $K$ increases, the theoretical loss becomes smoother, with the limit at $\boldpi=(0,1)$ corresponding to the proposed theoretical loss for consistent soft classification described in Section~\ref{framework:margin}. Additionally, note that while the loss functions, $\ell^+_{\boldpi}$ and $\ell^-_{\boldpi}$, are symmetric in Panels A and C of Figure~\ref{fig:steplossEg-01}, the same is not true for the loss functions in Panel B. This is due to the fact that the boundaries of interest, $\boldpi$, are symmetric between the two classes, i.e. $\boldpi = 1-\boldpi$, when $\boldpi=\{0.5\}$ and $\boldpi=(0,1)$, but not when $\boldpi=\{0.2, 0.4, 0.6\}$.
\begin{figure}[t!]
	\centering
	\includegraphics[width=\textwidth]{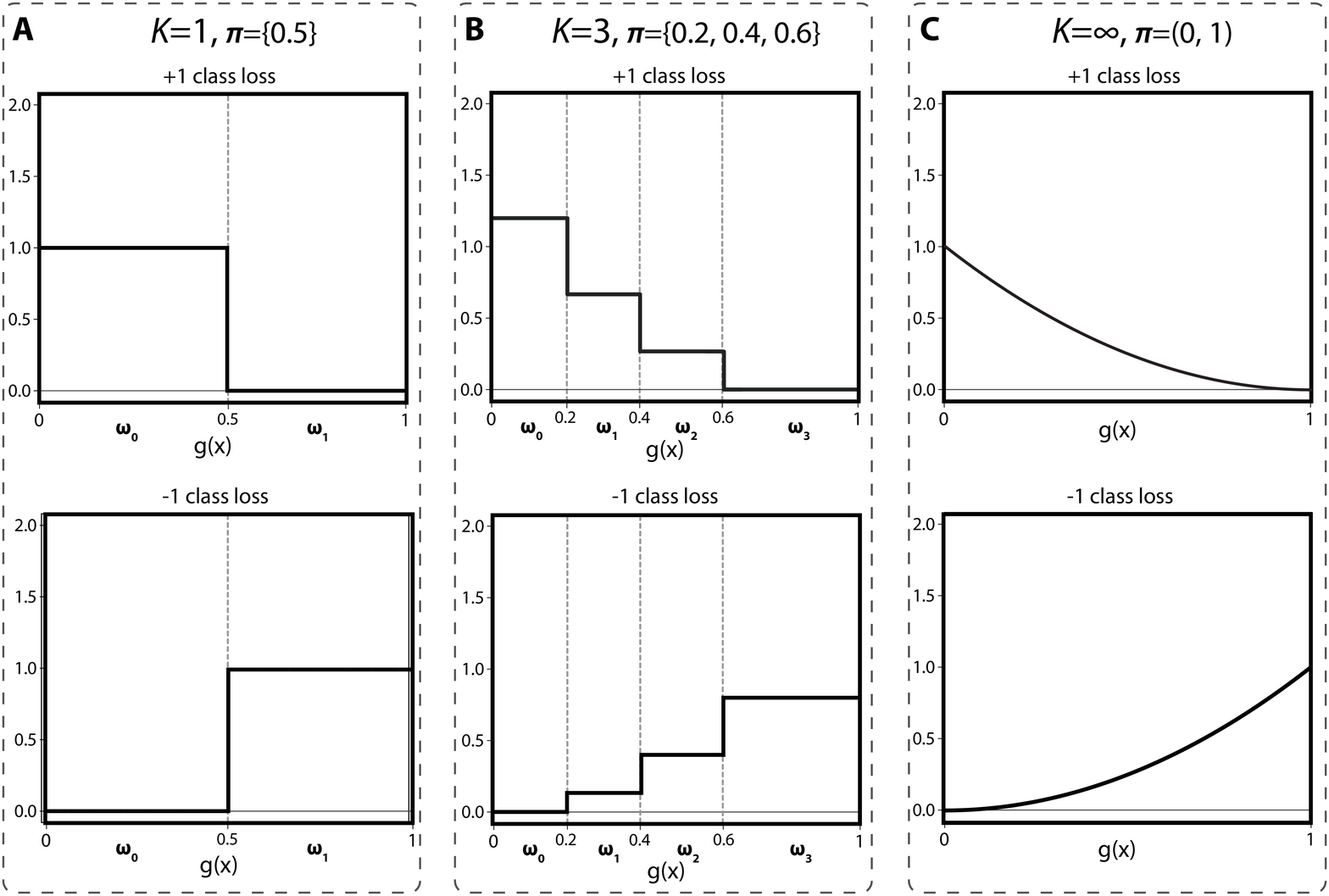}%
	\caption[Example of theoretical loss.]{Examples of the theoretical loss functions, $\ell^Y_{\boldpi}$, for observations from the positive and negative classes over $g(\bx)$ are shown for boundaries, $\boldpi$ at (A) $\{0.5\}$, (B) $\{0.2,0.4,0.6\}$, and (C) $(0, 1)$, corresponding to the problems shown in Figure~\ref{fig:cutEg}. The theoretical loss generalizes the standard 0$-$1 loss given in (A) by incorporating $K$ steps. As $K$ increases and the problem approaches soft classification, the theoretical loss becomes noticeably smoother.}
	\label{fig:steplossEg-01}
\end{figure}

The following result states that the class of problems defined by our theoretical loss indeed corresponds to the proposed framework of learning tasks. That is, the Bayes optimal rule given by $\bayesint(\bX) = \argmin{g} \mathbb{E}_{Y|\bX} \big\{ \ell^Y_{\boldpi} (g(\bX))\big\}$, is precisely the partitioning task described above. 
\begin{theorem} \label{thm:bayesrule}
	For fixed $K$ and $\boldpi$ defined as above, the Bayes optimal rule for the theoretical loss~(\ref{eq:general01}) is given by:%
	\begin{align*}
		\bayesint(\bX)
		&= \argmin{g\in\mathcal{G}_{\boldpi}} \mathbb{E}_{Y|\bX} \big\{\ell^Y_{\boldpi}(g(\bX)) \big\} \\
		&= \sum_{k=0}^K \omega_k \cdot \textbf{I}\{ p(\bX) \in \omega_k \} .
	\end{align*}
\end{theorem}

In addition to the results of Theorem~\ref{thm:bayesrule}, the theoretical loss functions for hard (\ref{eq:01loss}), rejection-option (\ref{eq:rejloss}), and weighted (\ref{eq:wloss}) classification can be derived as special cases of~(\ref{eq:general01}). This is shown by first noting the equivalence of $\intervals$ to $\calY$ and $\calY_{rej}$ based on the Bayes optimal rules, (\ref{eq:bayesrejection}) and (\ref{eq:bayesweighted}). From this equivalence, (\ref{eq:rejloss}) and (\ref{eq:wloss}) can be obtained directly from (\ref{eq:general01}). For soft classification, we derive a new theoretical loss from the limiting form of (\ref{eq:general01}):
\begin{align*}
	\ell^Y_{\boldpi}(g(\bX)) 
		&= \lim_{K\rightarrow\infty} \frac{2}{K} \sum_{k=1}^K \ell^Y_{\pi_k} (g(\bX)), \\
		&= \big(\textbf{I}\{Y=+1\} - g(\bX)\big)^2.
\end{align*}
The resulting theoretical loss is shown in Figure~\ref{fig:steplossEg-01}C. Since $\intervals = (0,1)$, the Bayes rule is simply the conditional class probability, $g(\bX) = p(\bX)$, corresponding to soft classification. All proofs, and a more complete derivation of these results may be found in the Supplementary Materials.

As with the problems described in Section~\ref{framework:margin}, optimization of $\ell_{\boldpi}$ with respect to $g\in\mathcal{G}_{\boldpi}$ is NP-hard. Thus, we first reformulate $\ell_{\boldpi}$ as a function on $\mathbb{R}$ to express the optimization over a collection of margin functions, $\mathcal{F}$. We then propose in Section~\ref{losses} to solve the approximate problem using convex surrogate loss functions. Generalizing the approach of \cite{Yuan2010a} for rejection-option classification, we frame the optimization task over the class of margin functions, $\mathcal{F}$, using a prediction rule $C:\mathbb{R}\times\mathbb{R}^K \rightarrow \intervals$ of the form:
\begin{align}
	C(f(\bx); \bolddelta) 
		&= \sum_{k=0}^K \omega_k \cdot \textbf{I}\{ f(\bx) \in (\delta_{k-1}, \delta_k] \}, \label{eq:deltarule}
\end{align}
for monotone increasing $\bolddelta = \{\delta_1,\ldots,\delta_K\}$, and $\delta_0 = -\infty$, $\delta_{K+1}=\infty$. Intuitively, each $\delta_k$ corresponds to the $\pi_k$-boundary along the range of the margin function, $f(\bX)$. As is common in margin-based learning, we write the theoretical loss as the following function over $Yf(\bX)$: 
\begin{align}
	L^Y_{\boldpi} (Yf(\bX); \bolddelta)
		&= \ell^Y_{\boldpi} (C(f(\bX); \bolddelta)) \notag \\
		&= 
		\begin{cases}
			\tfrac{2}{K} \sum_{k=1}^K (1-\pi_k) \cdot \textbf{I}\{ Yf(\bX) \leq \delta_k \} &\text{ if } Y=+1 \\
			\tfrac{2}{K} \sum_{k=1}^K \pi_k \cdot \textbf{I}\{ Yf(\bX) < -\delta_k \} &\text{ if } Y=-1
		\end{cases}. \label{eq:general01margin}
\end{align}
\begin{figure}[t]
	\centering
	\includegraphics[width=\textwidth]{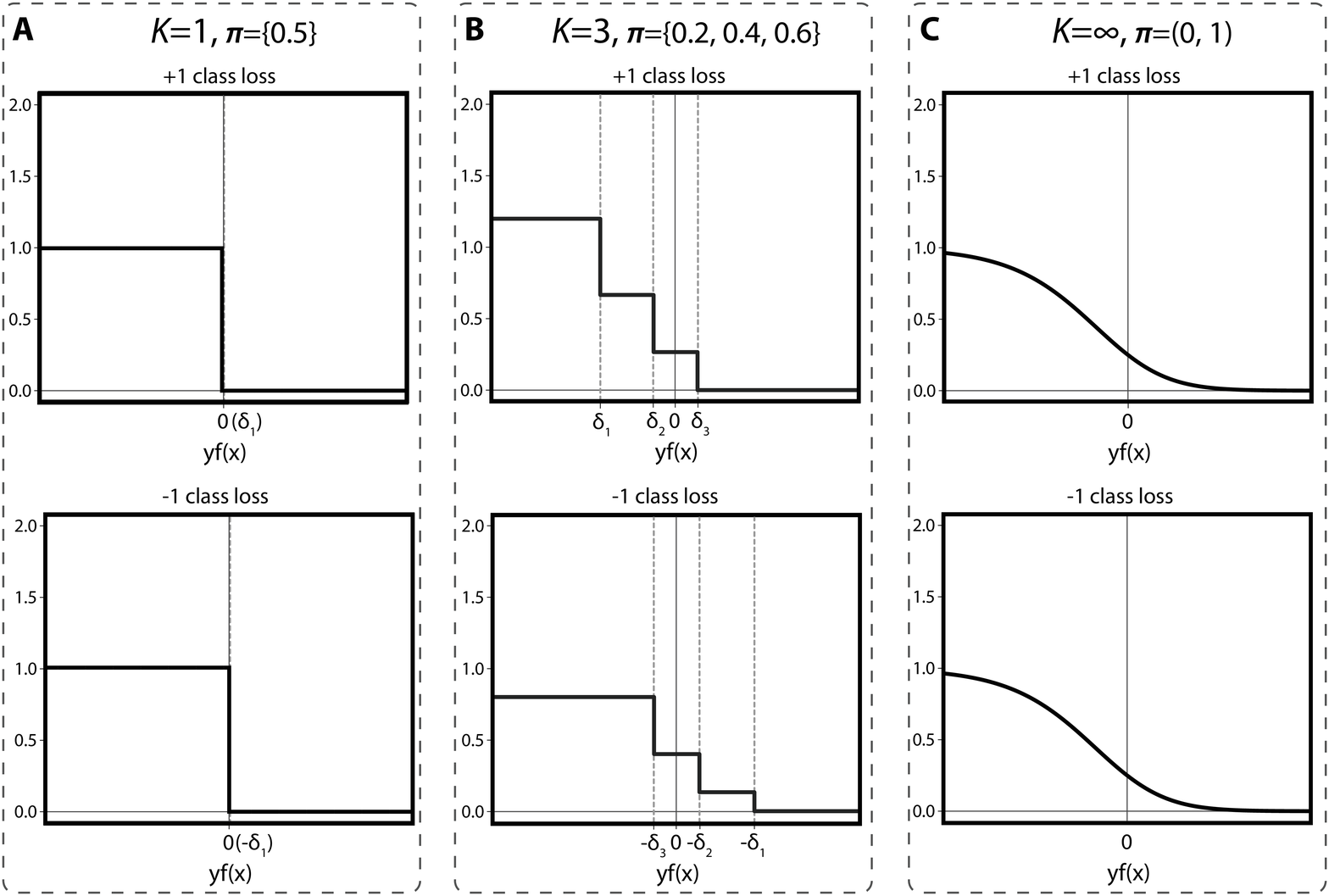}%
	\caption[Examples of the theoretical loss.]{Examples of the margin-based formulation of the theoretical loss function, $L^Y_{\boldpi}(\cdot\ ,\bolddelta)$, for observations from the positive and negative classes over $yf(\bx)$ are shown for boundaries, $\boldpi$, at (A) $\{0.5\}$, (B) $\{0.2, 0.4, 0.6\}$, and (C) $(0,1)$, using well-chosen $\bolddelta$.}
	\label{fig:steplossEg}
\end{figure}

In Figure~\ref{fig:steplossEg}, we plot the corresponding margin-based formulations of the theoretical loss functions shown in Figure~\ref{fig:steplossEg-01}, with well chosen $\bolddelta$. Intuitively, both $L^+_{\boldpi}(\cdot\ ; \bolddelta)$ and $L^-_{\boldpi}(\cdot\ ; \bolddelta)$ are non-increasing on $yf(\bx)$. We also note that $\ell^-_{\boldpi}$ and $L^-_{\boldpi}(\cdot\ ;\bolddelta)$ differ by a reflection along the vertical axis since $L^-_{\boldpi}(\cdot\ ;\bolddelta)$ is defined with respect to $yf(\bx)=-f(\bx)$. Given the margin-based formulation (\ref{eq:general01margin}), we propose to solve our class of problems using convex surrogate loss functions. In the following section, we first present necessary and sufficient conditions for a surrogate loss to be consistent to (\ref{eq:general01}). We then introduce a class of consistent piecewise linear surrogates, which includes the SVM hinge loss as a special case.

\section{Convex Surrogate Loss Functions}
\label{losses}
%
Since the proposed theoretical loss function (\ref{eq:general01}) and its margin-based reformulation (\ref{eq:general01margin}) are discontinuous and non-convex for any finite choice of $K$ and $\boldpi$, empirical minimization can quickly become intractable. Therefore, we propose to instead minimize a convex surrogate loss over the class of margin functions, as in hard and soft classification. In this section, we first provide necessary and sufficient conditions for a surrogate loss to be consistent for (\ref{eq:general01}) with fixed $K$ and $\boldpi$. Then, we introduce a class of convex piecewise linear surrogates which includes the SVM hinge loss as a special case. Intuitively, the piecewise linear surrogates each consist of $K$ non-zero segments, corresponding to the $K$ boundaries, $\boldpi$. In the limit, as $\boldpi$ becomes dense on $(0,1)$, the piecewise linear surrogate tends towards a smooth loss, as in Panel C of Figures~\ref{fig:steplossEg-01} and~\ref{fig:steplossEg}.

\subsection{Consistency}
\label{losses:consistency}
Throughout this section, we assume $K$ and $\boldpi$ to be fixed. First, let $\surrplus$ and $\surrminus$ denote a pair of convex surrogate loss functions for $\ell^+_{\boldpi}$ and $\ell^-_{\boldpi}$. Further, let $\bayessurr = \argmin{f} \mathbb{E}_{Y|\bX} \{ \surrogate^Y (Yf(\bX)) \}$ denote the $\phi^Y$-optimal rule over the class of all measurable functions. We call $\phi^Y$ consistent if there exists $\bolddelta\in\bbR^K$ such that the prediction rule (\ref{eq:deltarule}) satisfies $C(\bayessurr(\bx); \bolddelta) = \bayesint(\bx)$, i.e. if there exists a known monotone mapping from the $\phi^Y$-optimal rule to the $K+1$ partition of $\calX$ to $\intervals$. The following result provides necessary and sufficient conditions for the consistency of the surrogate loss $\phi^Y$ to $\ebpi^Y$.
\begin{theorem} \label{thm:consistency}
	A pair of convex surrogate loss functions, $\surrogate^Y$, are consistent for $\ell^Y_{\boldpi}$ if and only if there exists $\bolddelta\in\mathbb{R}^K$ such that for each $k=1,\ldots,K$: $\phi^{+\prime}(\delta_k)$ and $\phi^{-\prime}(-\delta_k)$ exist, $\phi^{+\prime}(\delta_k)$ and $\phi^{-\prime}(-\delta_k)<0$, and
	\begin{align}
		\frac{\phi^{-\prime}(-\delta_k)}{\phi^{-\prime}(-\delta_k)+\phi^{+\prime}(\delta_k)} = \pi_k. \label{eq:thmconsistency}
	\end{align} \vspace{-5pt}
\end{theorem}

Naturally, any surrogate loss satisfying the conditions of Theorem~\ref{thm:consistency} for some $\boldpi$, must also satisfy the set of conditions for any subset of the boundaries, $\boldpi^\prime \subseteq \boldpi$. Thus, for surrogate loss functions consistent for soft classification, i.e. when $\boldpi=(0,1)$, there exists an appropriate $\bolddelta$ for any possible $K$ and $\boldpi$. Similar intuition is used to justify the use of soft classification based plug-in classifiers described in Section~\ref{intro}. Examples of surrogate losses consistent for soft classification include the logistic, squared hinge, exponential, and DWD losses. Values of $\delta_k$ such that the conditions of Theorem~\ref{thm:consistency} are met for these loss functions are provided in Corollaries 3-8 of \cite{Yuan2010a}. In the next section, we introduce a class of piecewise linear surrogates which, similar to the SVM loss for hard classification, satisfy consistency for the $\boldpi$ of interest, but not for any $\boldpi^\prime \supset \boldpi$. We refer to such a piecewise linear surrogate as being minimally consistent for a corresponding set of boundaries, $\boldpi$. In contrast to soft classification losses which satisfy consistency for all $\boldpi\subseteq(0,1)$, minimally consistent surrogates are well-tuned for a given $\ell^Y_{\boldpi}$, and may provide improved stratification of $\calX$ to the sets, $\intervals$.

\subsection{Piecewise Linear Surrogates}
\begin{figure}[t]
	\centering
	\includegraphics[width=\textwidth]{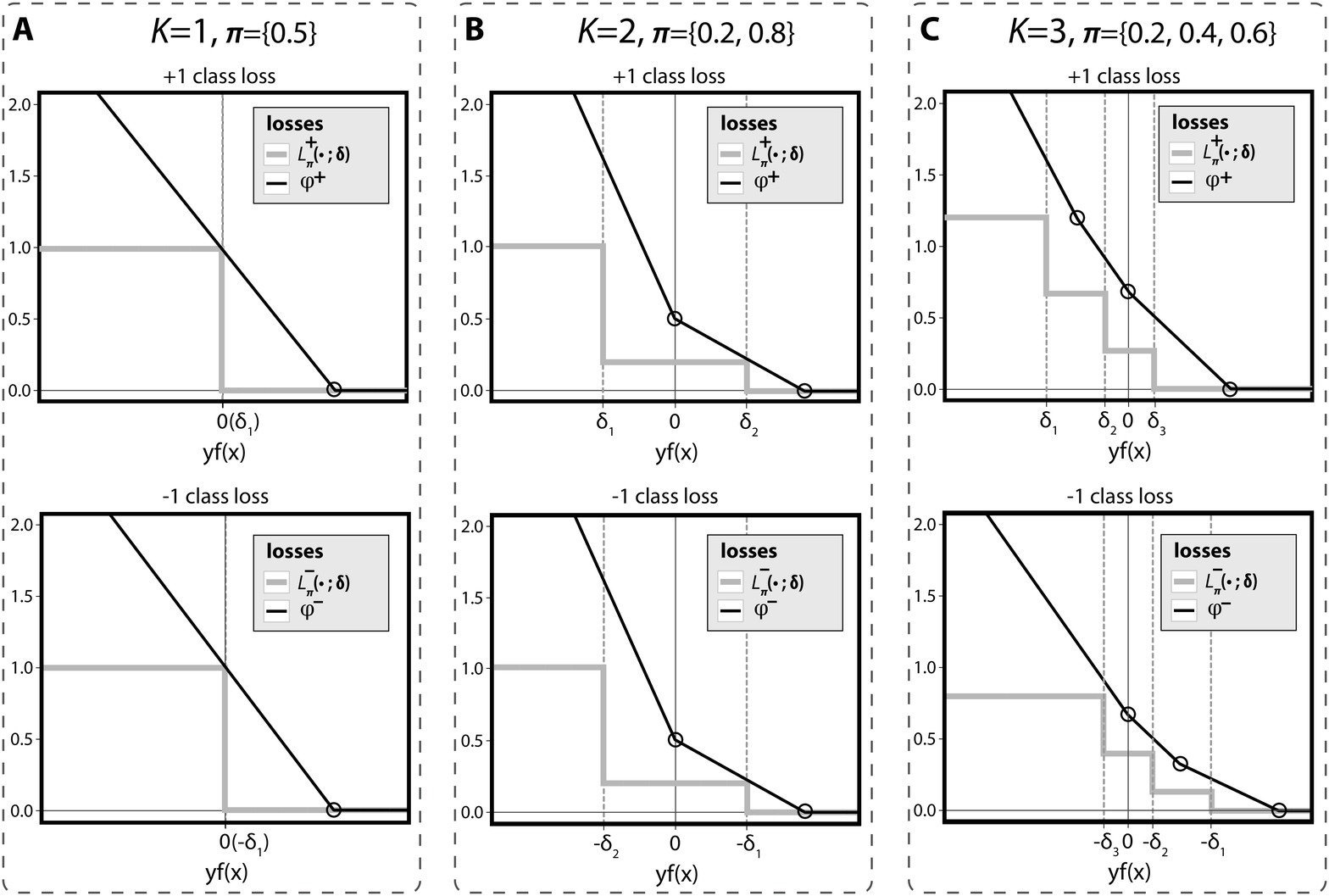}%
	\caption[Examples of piecewise linear surrogate losses.]{Examples of piecewise linear surrogates are shown along with the corresponding theoretical loss, $L_{\boldpi,\bolddelta}$ for (A) $\boldpi = \{0.5\}$ (hard classification), (B) $\boldpi = \{0.2, 0.8\}$ (rejection-option classification), and (C) $\boldpi = \{0.2, 0.4, 0.6\}$.}
	\label{fig:pieceSurrogates}
\end{figure}
Throughout, we use $\varsurrplus$ and $\varsurrminus$ to denote piecewise linear surrogates. To build intuition, in the columns of Figure~\ref{fig:pieceSurrogates}, we show examples of $\varsurrogate^Y$ for $K=1,2,3$, corresponding to hard classification, rejection-option classification, and the new problem shown in Figure~\ref{fig:cutEg}B. Circles are used to highlight the hinges, i.e. non-differentiable points, along the piecewise linear loss functions. The corresponding margin-based theoretical loss, $L^Y_{\boldpi}(\cdot\ ;\bolddelta)$, is also shown in each panel using appropriately chosen $\bolddelta$. First, note that the losses in Panels A and B of Figure~\ref{fig:pieceSurrogates} correspond to the standard SVM hinge loss and generalized hinge loss of \cite{Bartlett2008}, respectively. Consider the new surrogate losses in Figure~\ref{fig:pieceSurrogates}C for boundaries at $\boldpi = \{0.2, 0.4, 0.6\}$. Note that $\varsurrplus$ and $\varsurrminus$ each consist of $K$ non-zero linear segments. Furthermore, each linear segment only spans a single $\delta_k$ or $-\delta_k$ for $\varsurrplus$ and $\varsurrminus$, respectively. We will refer to these pairs of linear segments as the $\pi_k$-consistent segments. This construction allows for the consistency of the surrogate loss for each $\pi_k\in\boldpi$ to be controlled separately by the $K$ pairs of $\pi_k$-consistent segments along the piecewise linear loss.

We formulate our collection of piecewise linear surrogate losses as the maximum of the $K$ linear segments and 0. Consider first the surrogate loss for positive observations, $\varsurrplus$. Using $A^+(\pi),\ B^+(\pi)$ to denote the intercept and slope of the $\pi_k$-consistent segment, we express the piecewise linear loss as: 
\begin{align}
	\varsurrplus(z) = \max\{0,\ A^+(\pi_1) + B^+(\pi_1)\cdot z,\ \ldots,\ A^+(\pi_K) + B^+(\pi_K)\cdot z\}. \label{eq:varsurrplus}
\end{align}
We similarly use $A^-(\pi)$ and $B^-(\pi)$ to denote the intercept and slope of the $\pi_k$-consistent segment for the negative class loss such that:
\begin{align}
	\varsurrminus(z) = \max\{0,\ A^-(\pi_1) + B^-(\pi_1)\cdot z,\ \ldots,\ A^-(\pi_K) + B^-(\pi_K)\cdot z\}. \label{eq:varsurrminus}
\end{align}
By construction, the resulting piecewise linear losses are non-negative, convex and continuous. While (\ref{eq:varsurrplus}) and (\ref{eq:varsurrminus}) define a general class of piecewise linear losses, we focus on a subset of {minimally consistent piecewise linear surrogates}. In the following theorem, we provide a set of sufficient conditions for a piecewise linear loss to be minimally consistent for a specified $\boldpi$.
\begin{theorem}\label{thm:piecewise}
Let $H^Y(\pi,\pi^\prime) = (A^Y(\pi) - A^Y(\pi^\prime)) \big/ (B^Y(\pi^\prime) - B^Y(\pi))$ denote the location of the hinges along the respective loss functions between consecutive boundaries, $\pi < \pi^\prime$. Then, $\varsurrogate^Y$ is a {minimally consistent piecewise linear surrogate} for $\boldpi$ if the intercept and slope parameters, $A^Y(\pi)$ and $B^Y(\pi)$, satisfy the following conditions:
\begin{enumerate}
	\item[(C1)] $B^+(\pi)$ is non-decreasing, and $B^-(\pi)$ is non-increasing in $\pi$.
	\item[(C2)] The hinge points are such that:
		\begin{align*}
			-H^-(\pi_{k-1},\pi_{k}) &= H^+(\pi_{k-1},\pi_{k}) \ \ \ \ \ \text{for $k=2, \ldots, K$} , \\ 
			H^+(\pi_{k-1},\pi_{k}) &< H^+(\pi_{k},\pi_{k+1}) \ \ \ \ \ \text{for $k=2, \ldots, K-1$} , \\ 
			{A^-(\pi_1)}/{B^-(\pi_1)} &< H^+(\pi_1, \pi_{2}), \\
			{A^+(\pi_K)}/{B^+(\pi_K)} &> H^-(\pi_{K-1},\pi_{K}).
		\end{align*}
	\item[(C3)] $B^+(\pi), B^-(\pi)$ satisfy:
		\begin{align*}
			\frac{B^-(\pi_k)}{B^-(\pi_k)+B^+(\pi_k)} = \pi_k \ \ \ \ \text{for $1\leq k\leq K$}.
		\end{align*}
\end{enumerate}
\end{theorem}
Conditions (C1) and (C2) guarantee that the linear segments are well-ordered and non-degenerate along $Yf(\bX)$ with appropriately aligned hinge points. Condition (C3) guarantees the consistency of $\varsurrogate^Y$ to the corresponding $\ell_{\boldpi}$. Most importantly, by aligning the hinge points, $-H^-(\pi_{k-1},\pi_{k})$ and $H^+(\pi_{k-1},\pi_{k})$, we ensure that there does not exist a $\delta\in\bbR$ such that (\ref{eq:thmconsistency}) is satisfied for any $\pi\not\in\boldpi$. Next, we present an approach to obtaining $A^Y(\pi)$ and $B^Y(\pi)$ which satisfy the conditions of Theorem~\ref{thm:piecewise} using the logistic loss as an example.

\subsection{Logistic Derived Surrogates}
\label{losses:logistic}
\begin{figure}[t]
	\centering
	\includegraphics[width=0.8\textwidth]{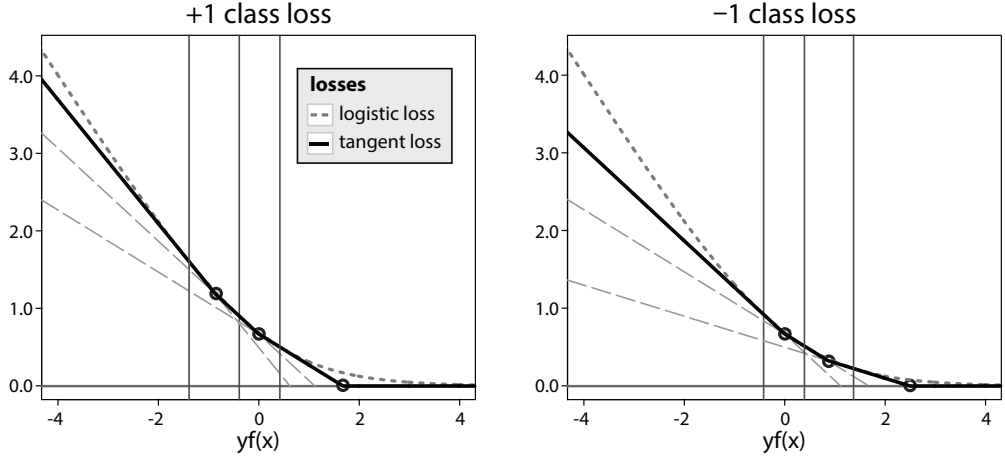}
	\caption[Examples of piecewise linear loss functions.]{A pair of piecewise linear loss functions, $\varsurrogate^Y$, obtained from the logistic loss for $\boldpi = \{0.2,0.4,0.6\}$ is shown along with the logistic loss (dotted lines), and the set of tangent lines used to derive $A^Y(\pi)$ and $B^Y(\pi)$ (dashed lines).}
	\label{fig:egPiecewise}
\end{figure}
In this section, we propose to construct piecewise linear losses by choosing $A^Y(\pi_k) + B^Y(\pi_k)\cdot z$ to be the tangent lines to the logistic loss at $Y\cdot\log(\tfrac{\pi_k}{1-\pi_k})$. A similar approach was used by \cite{Grandvalet2008} to construct a piecewise linear loss for the rejection-option problem. The following Proposition states that piecewise linear loss functions constructed using this approach satisfy the conditions of Theorem~\ref{thm:piecewise} for any choice of $K$ and $\boldpi$.
\begin{proposition}\label{prop:logisticpiece}
	For fixed $K$ and $\boldpi$, let $\varsurrogate^Y$ be the piecewise linear loss constructed from the tangent lines to the logistic loss such that $A^Y(\pi)$ and $B^Y(\pi)$ are defined as:
	\begin{align*}
		A^+(\pi) &= A^-(1-\pi) = -\pi\log(\pi) - (1-\pi)\log(1-\pi), \\
		B^+(\pi) &= B^-(1-\pi) = -(1-\pi).
	\end{align*} 
	Then, $\varsurrogate^Y$ is a minimally consistent piecewise linear surrogate for $\boldpi$ satisfying the conditions of Theorem~\ref{thm:piecewise}.
\end{proposition}
In Figure~\ref{fig:egPiecewise}, we illustrate the logistic-derived piecewise linear loss for $\boldpi = \{0.2,0.4,0.6\}$. The logistic loss is shown by dotted lines, with the piecewise linear surrogate functions for the positive and negative classes shown in solid black. Thin vertical lines are used to denote the tangent points where the losses are equal, and thin dashed lines give the tangent lines to the logistic loss corresponding to $A^Y(\pi_k)+B^Y(\pi_k)\cdot yf(\bx)$ for $\pi_k \in \boldpi$. Additionally, the non-differentiable hinge points are highlighted by circles. While the loss functions appear roughly equivalent within the region of the tangent points, the difference is non-negligible above and below these bounds. Notably, the piecewise linear losses diverge slower as $yf(\bx)$ tends to $-\infty$, suggesting the losses may be more robust to outliers \cite{Liu2011a}. Additionally, the logistic derived loss functions provide a natural spectrum for comparing the impact of targeting different partitions, $\intervals$, on the same dataset. We explore these issues using simulation in Section~\ref{simulations}.

\section{Statistical Properties}
\label{statprops}
%
We next derive statistical properties for surrogate loss functions to the theoretical loss, $\ell^Y_{\boldpi}$. In Subsection~\ref{statprops:risk}, we first show that the excess risk with respect to $\ell^Y_{\boldpi}$ may be bounded by the excess risk of a consistent surrogate loss. Then, in Subsection~\ref{statprops:rates}, we use these risk bounds to derive convergence rates for the empirical minimizer of a surrogate loss to the Bayes optimal rule. Our results generalize and extend those derived for the particular case of rejection-option classification in \cite{Herbei2006,Bartlett2008,Yuan2010a}, to an arbitrary number of boundaries.

\subsection{Excess Risk Bounds}
\label{statprops:risk}
For a rule $g\in\mathcal{G}_{\boldpi}$, we define the $\ell^Y_{\boldpi}$-risk of $g$ to be the expected loss of the rule, denoted by $R(g) = \mathbb{E}_{Y,\bX} \{ \ell_{\boldpi}^Y(g(\bX))\}$. In statistical machine learning, a natural measure of the performance of a rule is its excess risk: $\Delta R(g) = R(g) - R(\bayesint)$, where $R(\bayesint) = \min_{g\in\mathcal{G}_{\boldpi}} R(g)$ such that $\Delta R(g) \geq 0$. In this section, we derive convergence rates on $\Delta R(g)$ for rules obtained using consistent surrogate loss functions. For a surrogate loss $\surrogate^Y$, we similarly define the $\phi$-risk and excess $\phi$-risk over the class of margin functions, $\mathcal{F}$, to be $Q(f) = \mathbb{E}_{Y,\bX} \{\phi^Y(Yf(\bX))\}$ and $\Delta Q(f) = Q(f) - Q(\bayessurr)$. To obtain convergence rates on $\Delta R(g)$, we first show that under certain conditions, the excess $\phi$-risk of a margin function $f$ can be used to bound the corresponding excess $\ell^Y_{\boldpi}$-risk of $g=C(f; \bolddelta)$. Using this bound, we then derive rates of convergence on $\Delta R(g)$ through rates of convergence on $\Delta Q(g)$. The following additional notation is used to denote excess conditional $\ell^Y_{\boldpi}$-risk and excess conditional $\surrogate$-risk:
\begin{align*}
	R_p(g) &:= \mathbb{E}_{Y|\bX} \{ \ell_{\boldpi}^Y(g(\bX)) \}, &
	Q_p(f) &:= \mathbb{E}_{Y|\bX} \{ \surrogate^Y(Yf(\bX)) \}, \\
		\Delta R_p(g) &:= R_p(g) - R_p(\bayesint), &
			\Delta Q_p(f) &:= Q_p(f) - Q_p(\bayessurr).
\end{align*}

In the following results, we provide conditions under which there exists some function, $\rho:\bbR\rightarrow\bbR$, such that $\rho(\Delta Q(f))$ can be used to bound the corresponding $\Delta R(C(f;\bolddelta))$. 
\begin{theorem} \label{thm:bound}
	Let $\surrogate^Y$ be a consistent surrogate loss for $\ell_{\pi}^Y$ satisfying the conditions for Theorem~\ref{thm:consistency} at $\bolddelta$. Furthermore, suppose there exist constants $C>0$ and $s\geq1$ such that for all $k$,
	\begin{align}
		|p(\bX) - \pi_k|^s &\leq C^s \Delta Q_p (\delta_k). \label{eq:bndcnd}
		\shortintertext{Then,}
		\Delta R\big(C(f; \bolddelta)\big) &\leq C[2\cdot\Delta Q(f)]^{1/s}. \notag
	\end{align}
\end{theorem}
The above bound may be tightened as in \cite{Yuan2010a} by the additional assumption:
\begin{align}
	\mathbb{P}\{|p(\bX) - \pi_k|\leq t\} \leq At^\alpha, \ \ \ k=1, \ldots, K, \label{eq:margincondition}
\end{align}
for some $\alpha\geq 0$, $A \geq 1$. The bound (\ref{eq:margincondition}) generalizes the margin condition introduced by \cite{Mammen1999} and used in \cite{Herbei2006}.
\begin{theorem} \label{thm:bound2}
	In addition to the assumptions of Theorem~\ref{thm:bound}, assume that there exists $\alpha \geq 0$ and $A\geq 1$, such that (\ref{eq:margincondition}) holds for $t \in [0, \min_k \{ \pi_k-\pi_{k-1}, \pi_{k+1}-\pi_k \})$. Then, for some $D$ depending on $A, \alpha$,
	\[\Delta R\big(C(f; \bolddelta)\big) \leq D\cdot \Delta Q(f)^{1/(s+\beta-\beta s)}\]
	where $\beta=\alpha/(1+\alpha)$.
\end{theorem}

Note that when $\alpha = 0$, Theorem~\ref{thm:bound2} provides the same bound as Theorem~\ref{thm:bound}. However, as $\alpha \rightarrow \infty$, the bound becomes tighter, with $1/(s+\beta+\beta s)$ limiting to $1$. While neither result depends explicitly on $\boldpi$, Theorem~\ref{thm:bound2} suggests that tighter bounds may be achieved by only targeting $\boldpi$ such that the margin condition is satisfied with large $\alpha$. This reiterates the motivating intuition for our proposed framework, in which we formalize a class of learning problems for settings where more information than hard classification is desired, but soft classification may not be appropriate.

Corresponding values of $C$ and $s$ for the exponential, logistic, squared hinge and DWD losses, are provided in Corollaries 13--16 of \cite{Yuan2010a}. In the following result, we derive values of $C$ and $s$ for our class of minimally consistent piecewise linear surrogates.
\begin{corollary} \label{thm:PLbounds}
	For minimally consistent piecewise linear loss, $\varsurrogate^Y$, defined as in (\ref{eq:varsurrplus}) and (\ref{eq:varsurrminus}) and satisfying the conditions of Theorem~\ref{thm:piecewise} for boundaries $\boldpi$, the inequality (\ref{eq:bndcnd}) is satisfied by $s = 1$ and 
	\begin{align*}
		C = \max \left\{ -\frac{\pi_k}{B^-(\pi_k) \cdot |\delta_k - H_j|} : k=1,\ldots,K;\ j=0,\ldots,K\right\},
	\end{align*}
	where $H_0$ is used to denote $A^-(\pi_1)/B^-(\pi_1)$, $H_j$ to denote $H^+(\pi_j, \pi_{j+1})$ for $j=2,\ldots,K-1$, and $H_K$ to denote $A^+(\pi_K)/B^+(\pi_K)$.
\end{corollary}
Consider now a sequence of margin functions, $\{f_n\}_{n\geq1}$. By Theorems~\ref{thm:bound} and~\ref{thm:bound2}, to show that the excess $\ell^Y_{\boldpi}$-risk, $\Delta R(C(f_n; \bolddelta))$, converges to 0 as $n\rightarrow\infty$, it suffices to show that $\Delta Q(f_n)\rightarrow0$ as $n\rightarrow\infty$. In the following results, we derive convergence rates for $\Delta R\big(C(\cdot;\bolddelta)\big)$ for the sequence of functions, $\{\hat{f}_n\}_{n\geq1}$, where $\hat{f}_n$ is used to denote the empirical minimizer of the surrogate loss over a training set of size $n$.

\subsection{Rates of Convergence}
\label{statprops:rates}
In this section, we derive convergence results for two classes of surrogate loss functions separately. We first consider Lipschitz continuous and differentiable surrogate loss functions which satisfy a modulus of convexity condition specified below. Examples of such loss functions include the exponential, logistic, squared hinge and DWD losses. We then separately consider the class of piecewise linear surrogates described in Section~\ref{losses}. 

Let $\surrogate^Y$ denote a Lipschitz continuous and differentiable surrogate loss function. Assume that the corresponding $\phi$-risk, $Q(\cdot)$, has modulus of convexity,
\begin{align}
	\delta(\epsilon) &= 
		\inf \left\{ \frac{Q(f)+Q(g)}{2} - Q\left(\frac{f+g}{2}\right) : \mathbb{E}[(f-g)^2(\bX)]\geq \epsilon^2\right\} \label{eq:MOC}
\end{align}
satisfying $\delta(\epsilon)>c\epsilon^2$ for some $c>0$. Furthermore, let $L<\infty$ denote the Lipschitz constant, such that $|\surrogate^y(\bx) - \surrogate^y(\bx^{\prime})| \leq L|x-x^{\prime}|$ for all $\bx,\bx^{\prime} \in \mathbb{R}$ and $y=+1,-1$. Letting $\mathcal{F}_B$ denote the class of uniformly bounded functions such that $|f|\leq B$ for all $f\in\mathcal{F}_B$, we use $N_n = N(\tfrac{1}{n}, L_{\infty}, \mathcal{F}_B)$ to denote the cardinality of the set of closed balls with radius $\tfrac{1}{n}$ in $L_\infty$ needed to cover $\mathcal{F}_B$. Finally, as stated above, let $\hat f_n = \argmin{f\in\mathcal{F}_B} \sum_{i=1}^n \surrogate^{y_i}(y_i f(\bx_i))$ denote the empirical minimizer of $\surrogate^Y$ over the training set $\{(\bx_i, y_i)\}_{i=1}^n$. For the following corollary, we make use of Theorem~18 from \cite{Yuan2010a} which provides a bound on the expected estimation error, $Q(\hat f_n) - \inf_{f\in\mathcal{F}_B} Q(f)$, for consistent loss functions satisfying the modulus of convexity condition stated above. Combining Theorem~18 of \cite{Yuan2010a} with the excess risk bounds of Theorems~\ref{thm:bound} and~\ref{thm:bound2}, we obtain the following result.
\begin{corollary}\label{thm:combined}
	If $\surrogate^Y$ satisfies the assumptions of Theorems~\ref{thm:consistency} and~\ref{thm:bound}, and has modulus of convexity (\ref{eq:MOC}) satisfying $\delta(\epsilon)>c\epsilon^2$ for some $c>0$, then with probability at least $1-\gamma$,
	\begin{align*}
	\Delta R\big(C(\hat f_n; \bolddelta)\big) \leq 
		C \cdot 2^{1/s} \left\{ \inf_{f\in\mathcal{F}_B} \Delta Q(f) + 
		\frac{3L}{n} + 8\left(\frac{L^2}{2c} + 
		\frac{LB}{3}\right) \frac{\log(N_n/\gamma)}{n}\right\}^{1/s}.
	\end{align*}
	Furthermore, if the generalized margin condition of Theorem~\ref{thm:bound2} holds, then with probability at least $1-\gamma$,
	\begin{align}\label{eq:surrogatebound}
	\Delta R\big(C(\hat f_n; \bolddelta)\big) \leq 
		D \left\{ \inf_{f\in\mathcal{F}_B} \Delta Q(f) + 
		\frac{3L}{n} + 8\left(\frac{L^2}{2c} + 
		\frac{LB}{3}\right) \frac{\log(N_n/\gamma)}{n}\right\}^{1/(s+\beta-\beta s)},
	\end{align}
	for constants $C,D>0$ defined as in Theorems~\ref{thm:bound} and~\ref{thm:bound2}.
\end{corollary}

From the bound on excess risk obtained in Corollary~\ref{thm:combined}, corresponding rates of convergence can be derived based on the cardinality, $N_n$, of the class of functions, $\mathcal{F}_B$. 

Due to the non-differentiability of the loss at hinge points, our class of piecewise linear surrogates do not satisfy the modulus of convexity condition (\ref{eq:MOC}). The following theorem provides separate convergence results for our class of minimally consistent piecewise linear surrogates. Again, we use $\mathcal{F}_B$ to denote a class of uniformly bounded functions, and let $\hat f_n = \argmin{f\in\mathcal{F}_B} \sum_{i=1}^n \varsurrogate^{y_i}(y_i f(\bx_i))$ denote the empirical minimizer of $\varsurrogate^Y$.
\begin{theorem}\label{thm:piececonv}
	If $\varsurrogate^Y$ is a minimally consistent piecewise linear loss satisfying the conditions of Theorem~\ref{thm:piecewise}, satisfying the generalized margin condition of Theorem~\ref{thm:bound2}, then with probability at least $1-\gamma$,  
	\begin{align*}
	\Delta Q(\hat f_n) \leq 
		 \frac{3L}{n} 
			+ \frac{4LB}{3} \cdot G(\gamma) + 
			\left( \left(\frac{4LB}{3}\cdot G(\gamma) \right)^2 + 
				8 \cdot B^\prime \cdot G(\gamma) \right)^{1/2},
	\end{align*}
	where $G(\gamma) = \log(N_n/\gamma)/n$, and $B^\prime>0$ is some constant depending on $B$, $\varsurrogate^Y$, and margin constants $A, \alpha$.
\end{theorem}
Combining Theorems~\ref{thm:bound},~\ref{thm:bound2}, and~\ref{thm:piececonv}, we obtain the following corollary.
\begin{corollary}\label{thm:piececombined}
	If $\varsurrogate^Y$ is a minimally consistent piecewise linear loss satisfying the assumptions of Theorems~\ref{thm:consistency},~\ref{thm:bound}, and~\ref{thm:bound2}, then with probability at least $1-\gamma$,
	\begin{align}\label{eq:varsurrogatebound}
	\Delta R\big(C(\hat f_n; \bolddelta)\big) \leq 
		D \left\{
			 \frac{3L}{n} 
			+ \frac{4LB}{3} \cdot G(\gamma) + 
			\left( \left(\frac{4LB}{3}\cdot G(\gamma) \right)^2 + 
				8 \cdot B^\prime \cdot G(\gamma) \right)^{1/2} \right\}^{1/(s+\beta-\beta s)},
	\end{align}
	for constants $C,D>0$ defined as in Theorems~\ref{thm:bound} and~\ref{thm:bound2}.
\end{corollary}

As in Theorem~\ref{thm:bound2}, while the convergence rate of Theorem~\ref{thm:piececonv} does not depend on $\boldpi$ explicitly, it does depend on the parameters of the margin condition (\ref{eq:margincondition}). Therefore, Theorem~\ref{thm:piececonv} further suggests the advantage of targeting $\boldpi$ for which the data show strong separation with large $\alpha$. Furthermore, in contrast to Theorem~18 of \cite{Yuan2010a} which provides a bound on the expected estimation error, Theorem~\ref{thm:piececonv} bounds the total $\varsurrogate^Y$-risk, including both the expected estimation error, and expected approximation error of the class of functions $\mathcal{F}_B$. As a result, while the bounds in Corollary~\ref{thm:combined} include the separate approximation error term, $\inf_{f\in\mathcal{F}_B} \Delta Q(f)$, the piecewise linear bound in Corollary~\ref{thm:piececombined}, does not.

Based on the bounds in (\ref{eq:surrogatebound}) and (\ref{eq:varsurrogatebound}), rates of convergence can be obtained as in \cite{Yuan2010a}. As an example, we consider the case when $\mathcal{F_B}$ is the class of linear combinations of decision stumps, $f_\lambda$,
\begin{align*}
	f_\lambda (x) = \sum_{j=1}^M \lambda_j f_j(x)
\end{align*} 
where $\sum_j |\lambda_j| \leq B$, and $|f_j| < 1$. By (\ref{eq:surrogatebound}) and (\ref{eq:varsurrogatebound}), the same rate, $(M\log n/n)^{1/(s+\beta-\beta s)}$, can be obtained as in \cite{Yuan2010a} for both classes of surrogate losses considered above.

\section{Computational Algorithm}
\label{computation}

For a piecewise linear surrogate, $\varsurrogate^Y$, and convex penalty, $J(f)$, the objective (\ref{eq:loss+penaltyCh}) is a non-differentiable convex problem. Several approaches have been proposed for solving the similar non-differentiable and convex SVM objective, most commonly by reformulating (\ref{eq:loss+penaltyCh}) as a quadratic program (QP) with $2n$ constraints. The penalized objective (\ref{eq:loss+penaltyCh}) with $\varsurrogate^Y$ may also be formulated as a QP with $(K+1)n$ constraints. However, as with the SVM problem, the complexity of the problem grows almost cubically with the number of constraints, making the problem computationally intensive for moderately large $K$ and $n$ \cite{Bottou}. We therefore propose a projected sub-gradient descent algorithm similar to the PEGASOS algorithm \cite{Shalev-Shwartz2010}. 

We first rewrite (\ref{eq:loss+penaltyCh}) with piecewise linear surrogate, $\varsurrogate^Y$ defined as in (\ref{eq:varsurrplus}) and (\ref{eq:varsurrminus}) as:
\begin{align}
	\min_{h,b}\ &\ \ \frac{1}{n} \sum_{i=1}^n \Big(\max_{k=1,\ldots,K} \{ A^{y_i}(\pi_k) + B^{y_i}(\pi_k) \cdot y_i (h(\bx_i) + b) \} \Big)_+ 
		+ \frac{\lambda}{2} \|h\|^2_{\mathcal{H}}, \label{eq:objective1}
\end{align}
where $(z)_+ = \max\{0,z\}$, and $\mathcal{H}$ is some Reproducing Kernel Hilbert Space (RKHS) with norm $\|\cdot\|_{\mathcal{H}}$ and corresponding kernel function $K:\calX\times\calX\rightarrow\mathbb{R}$. Commonly, the margin function is formulated with a non-penalized intercept parameter, $b$. A more complete review of RKHS may be found in \cite{Aronszajn1950,Wahba1999}. In margin-based learning, kernel methods are commonly used to estimate non-linear classification boundaries. In the case of linear learning, i.e. $h(\bx) = \langle \bw, \bx \rangle$ for $\bw \in \mathbb{R}^p$, the penalty $\|h\|^2_{\mathcal{H}}$ reduces to $\|\bw\|^2$ and (\ref{eq:objective1}) may be written as:
\begin{align*}
	\min_{\bw,b}\ &\ \ \frac{1}{n} \sum_{i=1}^n \Big(\max_{k=1,\ldots,K} \{ A^{y_i}(\pi_k) + B^{y_i}(\pi_k) \cdot y_i (\langle \bw, \bx_i \rangle + b) \} \Big)_+ 
		+ \frac{\lambda}{2} \|\bw\|^2.
\end{align*}
We next describe our iterative algorithm for the linear learning setting. Let $\bw^{(m)}$ and $b^{(m)}$ denote the estimated parameters at the $m$-th iteration. Furthermore, at each iteration, let $B_i^*$ denote the sub-gradient of $\varsurrogate^{y_i}$ at $\langle \bw^{(m)}, \bx_i\rangle + b^{(m)}$ for $i=1, \ldots, n$. Using a decreasing step-size parameter, $\eta_m = (\lambda m)^{-1}$, we iterate the following updates until $\bw^{(m)}$ and $b^{(m)}$ converge:
\begin{enumerate}
	\item $\bw^{(m)} = \bw^{(m-1)} + \eta_m (\tfrac{1}{n}\sum_i B_i^* y_i \bx_i - \lambda \bw^{(m-1)})$,
	\item $b^{(m)} = b^{(m-1)} + \eta_m (\tfrac{1}{n}\sum_i B_i^* y_i)$, 
	\item $[\bw^{(m)}, b^{(m)}] = \min \{ 1, \tfrac{\lambda^{-1/2}}{\| [\bw^{(m)}, b^{(m)}] \|} \} [\bw^{(m)}, b^{(m)}]$,
\end{enumerate}
where $B_i^*$ is used to denote the sub-gradient of $\varsurrogate^{y_i}$ at $y_i (\langle \bx_i, \bw\rangle + b)$. The final projection step is included to ensure $\| [\bw^{(m)}, b^{(m)}] \|^2 \leq \lambda^{-1}$ at each iteration \cite{Calamai1987,Shalev-Shwartz2010}. In the following section, we apply our projected sub-gradient descent algorithm to simulated datasets to illustrate the utility of our class of problems.

\section{Simulations}
\label{simulations}
In this section, we use simulations to illustrate the performance achieved by targeting different binary learning problems. Namely, we compare the performance of several minimal consistent piecewise linear losses against the standard logistic classifier, when the underlying conditional class probability, $p(\bX)$, is piecewise constant. Piecewise linear loss functions are derived from the logistic loss as described in Section~\ref{losses:logistic}, and the sets of boundaries, $\bolddelta$, are chosen by the tangent points to the logistic loss. In each simulation, we consider piecewise linear losses with $\boldpi_1 = \{1/2\}$, $\boldpi_2 = \{1/3, 2/3\}$, and $\boldpi_3 = \{1/4, 2/4, 3/4\}$. All methods are tuned over a grid of penalty parameters $\lambda \in \{2^{-15}, 2^{-14}, \ldots, 2^{10}\}$, using training and tuning sets of 100 observations each. Piecewise linear classifiers and the logistic classifier are tuned with respect to the correspond theoretical loss (\ref{eq:general01}) and likelihood function, respectively. The performance of each estimated model is evaluated using a test set of 10,000 observations. Each simulation was replicated 100 times.

\subsection{Simulation 1} \label{simulations:sim1}
\begin{figure}[t]
	\centering
	{\includegraphics[width=\textwidth]{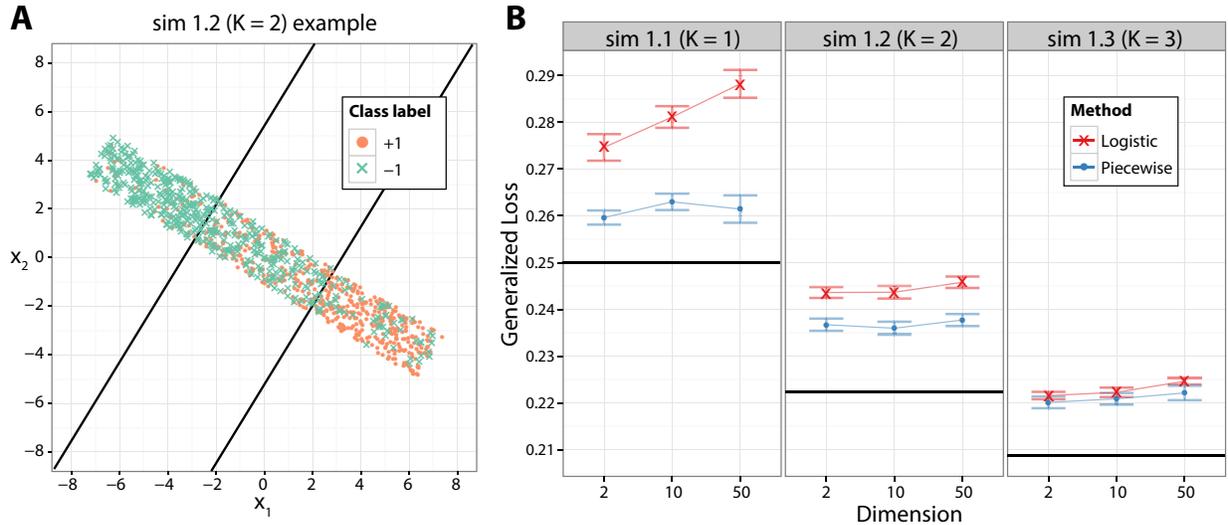}}
	\caption[Simulation~1 results]{(A) Sample dataset of 1000 observations drawn from the generating distribution for Simulation~1.2. The two Bayes optimal boundaries separating the three regions of constant $p(\bX)$ are shown with black lines. (B) Comparison of the performance of the piecewise linear and logistic classifiers for the three settings of Simulation~1 across varying dimension. In each panel, the median loss and standard error over 100 replications is shown along with the Bayes minimal loss in black.} 
	\label{fig:sim1}
\end{figure}
In this setting, data are simulated uniformly from $[-8, 8] \times [-1, 1]^{p-1}$ for $p=2, 10, 50$, subject to a random rotation in the $p$-dimensional space. We consider three variations of this setting, in which the data were simulated with underlying conditional class probability, defined with respect to the sampling space prior to rotation: 
\begin{enumerate}
	\item[1.1] $p(\bX) = \tfrac{1}{4} \textbf{I}\{x_1 \in  [-8, 0)\} + \tfrac{3}{4} \textbf{I}\{x_1 \in [0, 8]\}$, 
	\item[1.2] $p(\bX) = \tfrac{1}{6} \textbf{I}\{x_1 \in  [-8, -\tfrac{8}{3})\} + \tfrac{3}{6} \textbf{I}\{x_1 \in  [-\tfrac{8}{3},\tfrac{8}{3})\} + \tfrac{5}{6} \textbf{I}\{x_1 \in  [\tfrac{8}{3}, 8]\} $,
	\item[1.3] $p(\bX) = \tfrac{1}{8} \textbf{I}\{x_1 \in  [-8, -4)\} + \tfrac{3}{8} \textbf{I}\{x_1 \in  [-4,0)\} + \tfrac{5}{8} \textbf{I}\{x_1 \in  [0, 4)\} + \tfrac{7}{8} \textbf{I}\{x_1 \in  [4, 8]\} $.
\end{enumerate}
Settings~1.1, 1.2, and 1.3 have one, two and three natural boundaries due to the piecewise constant form of $p(\bX)$. In Figure~\ref{fig:sim1}A, we show 1000 observations drawn from simulation setting~1.2, with observations from the positive and negative class shown in orange and green. The Bayes optimal boundaries are also shown in black. For settings~1.1, 1.2, and 1.3, we use the piecewise linear losses with boundaries at $\boldpi_1$, $\boldpi_2$, and $\boldpi_3$, respectively. In each setting, the performance of the piecewise linear and logistic classifiers is evaluated using the theoretical loss for boundaries at $\boldpi_1$, $\boldpi_2$, and $\boldpi_3$. In these simulations, we aim to illustrate the advantage of minimizing and tuning with respect to an appropriate theoretical loss, which matches the underlying form of the data.

The results are shown in Figure~\ref{fig:sim1}B, along with the Bayes minimal loss, which provides a lower bound on the theoretical loss in each setting. In all settings, the piecewise linear classifier outperforms the logistic classifier, with the improvement decreasing as the number of boundaries, $K$ increases. This makes intuitive sense, as the piecewise linear loss converges to the logistic loss as $K\rightarrow\infty$. The most significant improvement is seen in setting 1.1, in which the piecewise linear classifier and theoretical loss correspond to the standard SVM and misclassification error. These results confirm previous results highlighting the advantage of hard classifiers over soft classifiers when the underlying $p(\bX)$ is piecewise constant \cite{Liu2011a}. Furthermore, the complete set of results illustrates the transition of this behavior as the number of boundaries increases.

\subsection{Simulation 2}\label{simulations:sim2}
\begin{figure}[t]
	\centering
	{\includegraphics[width=\textwidth]{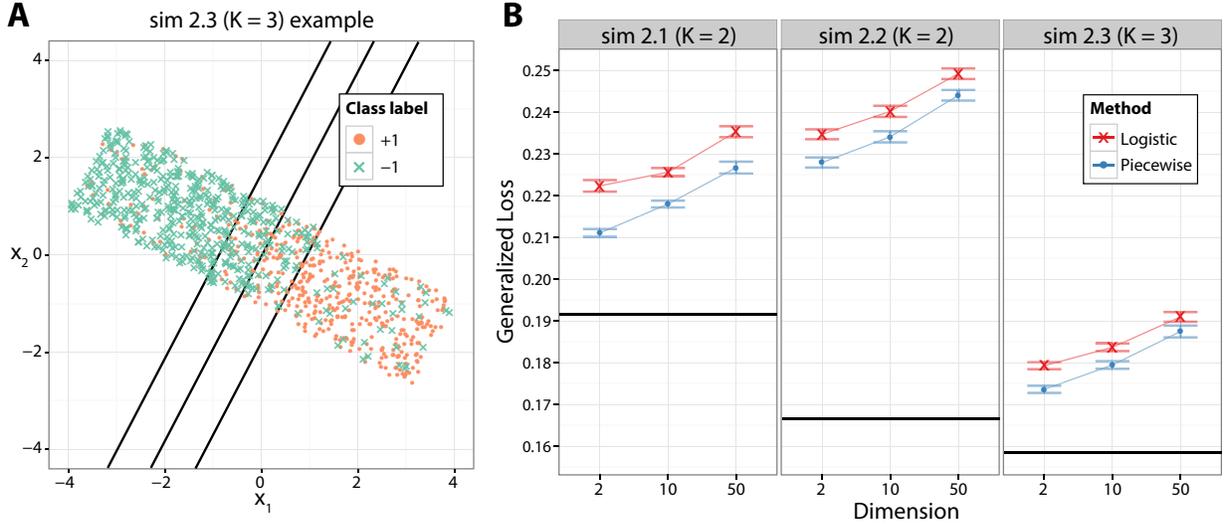}}
	\caption[Simulation~2 results]{(A) Sample dataset of 1000 observations drawn from the generating distribution for Simulation~2.3. The three Bayes optimal boundaries separating the four regions of constant $p(\bX)$ are shown with black lines. (B) Comparison of the performance of the piecewise linear and logistic classifiers for the three settings of Simulation~2 across varying dimension. In each panel, the median loss and standard error over 100 replications is shown along with the Bayes minimal loss in black.}
	\label{fig:sim2}
\end{figure}
In Simulation~1, the piecewise constant regions of $p(\bX)$ were of equal size. In our second set of simulations, we consider unequally spaced conditional class probabilities. Observations were uniformly sampled over $[-4, 4] \times [-1, 1]^{p-1}$, for $p=2, 10, 50$, again subject to a random rotation. The following conditional class probabilities were considered, again, with respect to the sampling space prior to rotation:
\begin{enumerate}
	\item[2.1] $p(\bX) = \tfrac{1}{6} \textbf{I}\{x_1 \in  [-4, -0.6)\} + \tfrac{3}{6} \textbf{I}\{x_1 \in  [-0.6, 0.6)\} + \tfrac{5}{6} \textbf{I}\{x_1 \in  [0.6, 4]\} $,
	\item[2.2] $p(\bX) = \tfrac{1}{6} \textbf{I}\{x_1 \in  [-4, -2)\} + \tfrac{3}{6} \textbf{I}\{x_1 \in  [-2, 0)\} + \tfrac{5}{6} \textbf{I}\{x_1 \in  [0, 4]\} $,
	\item[2.3] $p(\bX) = \tfrac{1}{8} \textbf{I}\{x_1 \in  [-4, -0.8)\} + \tfrac{3}{8} \textbf{I}\{x_1 \in  [-0.8,0)\} + \tfrac{5}{8} \textbf{I}\{x_1 \in  [0, 0.8)\} + \tfrac{7}{8} \textbf{I}\{x_1 \in  [0.8, 4]\} $.
\end{enumerate}
In settings~2.1 and~2.3, we consider $p(\bX)$ with heavy tails, and in setting~2.2, we consider the case with asymmetric $p(\bX)$. A sample of 1000 observations drawn from setting~2.3 is shown in Figure~\ref{fig:sim2}A, with the Bayes optimal boundaries in black.  For settings~2.1, 2.2, and 2.3, we use the piecewise linear losses with boundaries at $\boldpi_2$, $\boldpi_2$, and $\boldpi_3$, respectively. The performance of the piecewise linear and logistic classifiers is again evaluated using the corresponding theoretical loss function. Simulation results are shown in Figure~\ref{fig:sim2}B. As in Simulation~1, the piecewise linear classifier outperforms the logistic classifier in all cases. Again, the improvement is greater in settings~2.1 and~2.2 than in setting~2.3, as the piecewise linear loss converges to the logistic loss with increasing $K$.

\section{ADNI Data Analysis}
\label{realdata}

In this section, we apply the proposed interval estimation procedure to a MRI dataset of healthy normal control (NC) and early Alzheimer's disease (AD) subjects. Data were obtained from the ADNI database (\url{adni.loni.usc.edu}). The ADNI was launched in 2003 by the National Institute on Aging (NIA), the National Institute of Biomedical Imaging and Bioengineering (NIBIB), the Food and Drug Administration (FDA), private pharmaceutical companies and non-profit organizations as a \$60 million, 5-year public-private partnership. The Principal Investigator of this initiative is Michael W.~Weiner, MD, VA Medical Center and University of California - San Francisco. ADNI is the result of efforts of many co-investigators from a broad range of academic institutions and private corporations, and subjects have been recruited from over 50 sites across the U.S. and Canada. For up-to-date information, see \url{www.adni-info.org}.

\begin{figure}[t]
	\centering
	{\includegraphics[width=\textwidth]{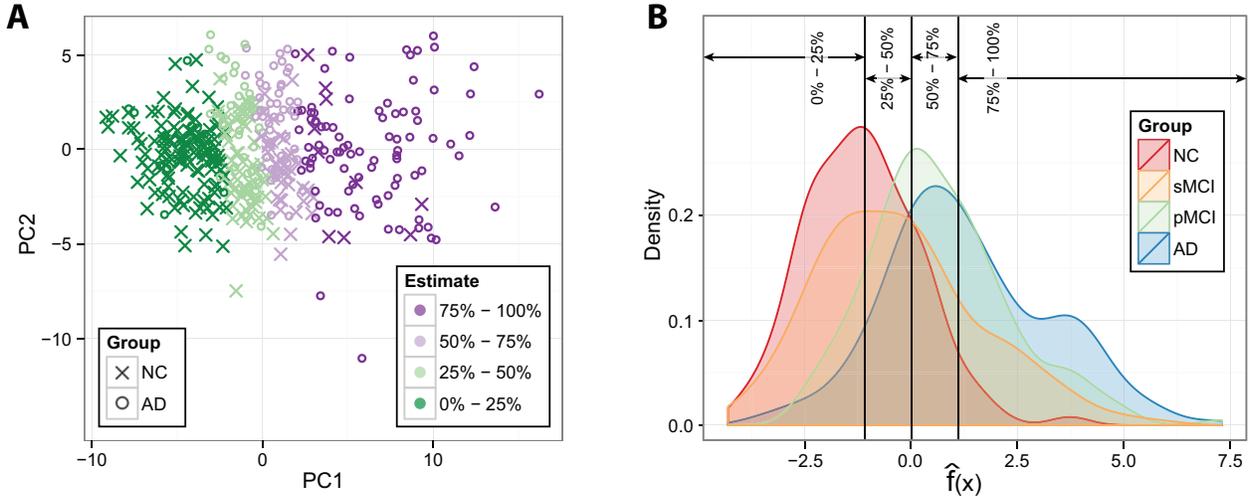}}
	\caption[Real data analysis]{Analysis of ADNI MRI dataset with $\boldpi = \{0.25, 0.5, 0.75\}$. (A) Scatterplot of first two PCs for AD and NC subjects colored by estimated interval. (B) Density plots of predicted $\hat{f}(\bx)$ for AD, NC, and two intermediary subject groups, sMCI and pMCI. Corresponding interval cutoffs are shown with vertical lines.}
	\label{fig:data}
\end{figure}

The dataset we use consists of 93 MRI features measured for 225 NC and 186 AD subjects, and was processed as described in \cite{Yu2014}. As in Section~\ref{simulations}, the logistic-derived piecewise linear loss is used to target the conditional class probability of AD at $\boldpi = \{1/4, 2/4, 3/4\}$. Two-fold cross validation is used to determine the optimal $\lambda$ over $\{2^{-15}, 2^{-14}, \ldots, 2^5\}$. The first two principal components (PCs) of the 411 NC and AD subjects are shown in Figure~\ref{fig:data}A, along with the estimated interval for each subject. Interestingly, the four distinct probability groups appear to separate along the first PC direction.

In addition to NC and AD subjects, the dataset also includes subjects with mild cognitive impairment (MCI), further classified as either progressive MCI (pMCI, 167 subjects) or stable MCI (sMCI, 226 subjects), depending on whether or not the subject progressed to develop AD during the study. The sMCI and pMCI may be considered as intermediary states between the NC and AD subjects. As such, in Figure~\ref{fig:data}B, we show the distribution of margin values, $\hat{f}(\bx)$, for NC, sMCI, pMCI, and AD subjects to investigate the transition between the four distinct groups. The corresponding interval boundaries are shown by vertical lines. Interestingly, while not well-differentiated, the four groups appear to peak within each of the four intervals, with the densities shifting in the expected order. Overall, our method appears to appropriately divide the subject according to the severity of the disease.

\section{Discussion}
\label{discussion}

Supervised learning tasks with a discrete class label are commonly encountered in practice. Several problems have been formally defined and studied within this context, including hard, soft, and rejection-option classification. In this paper, we introduce a unified framework of binary learning tasks targeting partial or complete estimation of the conditional class probability, $p(\bX)$, which encompassing these problems. In contrast to previous frameworks connecting hard and soft classification, our approach spans a space of learning problems, rather than specific loss functions or classification methods. Our approach thus provides a unique perspective to study the transition between hard and soft classification.

We formalize our family of binary learning problems through a unified theoretical loss (\ref{eq:general01}), a corresponding margin based relaxation (\ref{eq:general01margin}), and a proposed class of minimally consistent piecewise linear surrogates. Simulation studies using the class of piecewise linear loss functions reinforce previous results on hard and soft classification, and illustrate the transitional behavior between the class of problems. Finally, an application of our interval estimation approach to a MRI dataset from the ADNI study further illustrates the utility of our proposed class of problems.


\vskip .25in
\begin{spacing}{1.3}
\noindent
{\bf Acknowledgments:} \\
The authors are supported in part by National Institutes of Health (NIH) Grants U24 CA143848 (Hayes) and U24 CA143848-02S1 (Kimes). Data collection and sharing for this project was funded by the Alzheimer's Disease Neuroimaging Initiative (ADNI) (NIH Grant U01 AG024904). ADNI is funded by the National Institute on Aging, the National Institute of Biomedical Imaging and Bioengineering, and through generous contributions from the following: Abbott; Alzheimer's Association; Alzheimer's Drug Discovery Foundation; Amorfix Life Sciences Ltd.; AstraZeneca; Bayer HealthCare; BioClinica, Inc.; Biogen Idec Inc.; Bristol-Myers Squibb Company; Eisai Inc.; Elan Pharmaceuticals Inc.; Eli Lilly and Company; F. Hoffmann-La Roche Ltd and its affiliated company Genentech, Inc.; GE Healthcare; Innogenetics, N.V.; IXICO Ltd.; Janssen Alzheimer Immunotherapy Research \& Development, LLC.; Johnson \& Johnson Pharmaceutical Research \& Development LLC.; Medpace, Inc.; Merck \& Co., Inc.; Meso Scale Diagnostics, LLC.; Novartis Pharmaceuticals Corporation; Pfizer Inc.; Servier; Synarc Inc.; and Takeda Pharmaceutical Company. The Canadian Institutes of Health Research is providing funds to support ADNI clinical sites in Canada. Private sector contributions are facilitated by the Foundation for the National Institutes of Health (www.fnih.org). The grantee organization is the Northern California Institute for Research and Education, and the study is coordinated by the Alzheimer's Disease Cooperative Study at the University of California, San Diego. ADNI data are disseminated by the Laboratory for Neuro Imaging at the University of Southern California. This research was also supported by NIH grants P30 AG010129 and K01 AG030514.
\end{spacing}


\bibliographystyle{ieeetr}
\bibliography{references}


\clearpage
\section*{SUPPLEMENTARY MATERIALS}
\label{supplement}
\beginsupplement

\section{Common learning problems as special cases}
In this section, we show that our class of problems encompasses hard, weighted, rejection-option, and soft classification. For hard, weighted and rejection-option classification, the equivalence is derived by showing that for specific choices of $\boldpi$, the theoretical loss (\ref{eq:general01}) reduces to the standard losses given in (\ref{eq:01loss}), (\ref{eq:rejloss}), and (\ref{eq:wloss}). For soft classification, the equivalence is shown by deriving the limiting form of the theoretical loss (\ref{eq:general01}), and showing that the limiting loss is optimized by $p(\bX)$.

\subsection{Hard and Weighted Classification}
For hard classification, let $\boldpi=\{0.5\}$, such that $\intervals = \{\omega_0, \omega_1\} = \{[0,0.5], (0.5,1]\}$. Then, the theoretical loss (\ref{eq:general01}) may be simplified as:
\begin{align}
	\ell^Y_{\boldpi} \big(g(\bX))\big) 
		&= 2 \cdot \ell^Y_{0.5} \big(g(\bX))\big) \notag \\
		&= 
		\begin{cases}
			2\cdot (1 - 0.5) \cdot \textbf{I}\{g(\bX) \leq 0.5\} &\text{if } Y = +1 \\
			2\cdot 0.5 \cdot \textbf{I}\{g(\bX) > 0.5\} &\text{if } Y = -1		
		\end{cases} \notag \\
		&= \textbf{I}\{\big(g(\bX) \leq 0.5,\ Y=+1\big) \text{ or } \big(g(\bX) > 0.5,\ Y=-1\big)\} \notag \\
		&= \textbf{I}\{\big(g(\bX) = \omega_0,\ Y=+1\big) \text{ or } \big(g(\bX)=\omega_1,\ Y=-1\big)\}. \label{eq:int01}
\end{align}
The equivalence of (\ref{eq:int01}) to the 0$-$1 loss (\ref{eq:01loss}) follows by noting that for the Bayes optimal rule for hard classification (\ref{eq:bayesweighted}), predictions of $\hatY = +1,-1$ correspond to $p(\bx) \in \omega_0$ and $p(\bx) \in \omega_1$, respectively. More generally, the weighted 0$-$1 loss (\ref{eq:wloss}) may be similarly recovered up to a multiplicative constant by letting $\boldpi = \{\pi\}$ for any $\pi \in (0,1)$ such that $\intervals = \{\omega_0, \omega_1\} = \{[0,\pi], (\pi,1]\}$, and
\begin{align}
	\ell^Y_{\boldpi} \big(g(\bX))\big) 
		&= 2 \cdot \ell^Y_{\pi} \big(g(\bX))\big) \notag \\
		&= 
		\begin{cases}
			2\cdot (1 - \pi) \cdot \textbf{I}\{g(\bX) \leq \pi_1\} &\text{if } Y = +1 \\
			2\cdot \pi \cdot \textbf{I}\{g(\bX) > \pi_1\} &\text{if } Y = -1		
		\end{cases} \notag \\
		&\propto (1 - \pi) \cdot \textbf{I}\{g(\bX)=\omega_0,\ Y=+1\} \notag \\
			&\ \ \ \ + \pi \cdot \textbf{I}\{g(\bX)=\omega_1,\ Y=-1\}. \label{eq:intw01}
\end{align}
Again, the equivalence of (\ref{eq:intw01}) to the weighted 0$-$1 loss (\ref{eq:wloss}) follows from the form of the Bayes optimal rule for weighted classification (\ref{eq:bayesweighted}).

\subsection{Rejection-Option Classification}
For rejection-option classification as formulated in [12], let $\boldpi = \{\pi, 1-\pi\}$ for some $\pi \in (0,0.5)$, such that $\intervals = \{\omega_0, \omega_1, \omega_2\} = \{[0,\pi], (\pi,1-\pi], (1-\pi,1] \}$. We first rewrite (\ref{eq:general01}) as:
\begin{align*}
	\ell^Y_{\boldpi} \big(g(\bX))\big) 
		&= \frac{2}{2} \sum_{k=1}^2 \ell^Y_{\pi_k} \big(g(\bX)\big) \\
		&= 
		\begin{cases}
			\sum_{k=1}^2 (1-\pi_k) \cdot \textbf{I}\{g(\bX) \leq \pi_k\} &\text{if } Y = +1 \\
			\sum_{k=1}^2 \pi_k \cdot \textbf{I}\{g(\bX) > \pi_k\} &\text{if } Y = -1
		\end{cases} \\
		&=
		\begin{cases}
			(1-\pi) \cdot \textbf{I}\{g(\bX) \leq \pi\} 
				+ \pi \cdot \textbf{I}\{g(\bX) \leq 1-\pi\} &\text{if } Y = +1 \\
			\pi \cdot \textbf{I}\{g(\bX) > \pi\} 
				+ (1-\pi) \cdot \textbf{I}\{g(\bX) > 1-\pi\} &\text{if } Y = -1
		\end{cases} \\
		&=
		\begin{cases}
			(1-\pi) \cdot \textbf{I}\{g(\bX)=\omega_0\} 
				+ \pi \cdot \textbf{I}\{g(\bX)\not=\omega_2\} &\text{if } Y = +1 \\
			\pi \cdot \textbf{I}\{g(\bX) \not=\omega_0\} 
				+ (1-\pi) \cdot \textbf{I}\{g(\bX)=\omega_2\} &\text{if } Y = -1
		\end{cases} \\
		&= 
		\begin{cases}
			1 &\text{if } (g(\bX)=\omega_0,\ Y=+1) \text{ or } (g(\bX)=\omega_2,\ Y=-1) \\
			\pi &\text{if } g(\bX) = \omega_1 \\
			0 &\text{otherwise}
		\end{cases}.
\end{align*}
The equivalence of (\ref{eq:general01}) to the rejection option loss (\ref{eq:rejloss}) is established by noting the correspondence between predictions of $\{\omega_0,\omega_1,\omega_2\}$ and $\hatY_{rej} = \{-1, 0, +1\}$ for the Bayes optimal rejection-option rule (\ref{eq:bayesrejection}).

\subsection{Soft Classification}
Although not traditionally formulated as the minimization of a theoretical loss, the soft classification problem may be derived as the special case of (\ref{eq:general01}) when $K=\infty$ and $\boldpi$ becomes dense on $(0,1)$, such that $\intervals = (0,1)$. The limiting form of (\ref{eq:general01}), which we define as the average of $K$ functions, may be expressed as the following integral:
\begin{align*}
	\ell^Y_{\boldpi}(g(\bX)) 
		&= \lim_{K\rightarrow\infty} \frac{2}{K} \sum_{k=1}^K \ell^Y_{\pi_k} (g(\bX)) \\
		&= 
		\begin{cases}
			2 \int_0^1 (1-\pi) \cdot \textbf{I}\{g(\bX) \leq \pi\} d\pi &\text{if } Y = +1 \\
			2 \int_0^1 \pi \cdot \textbf{I}\{g(\bX) > \pi\} d\pi &\text{if } Y = -1
		\end{cases} \\
		&=
		\begin{cases}
			(1 - g(\bX))^2 &\text{if } Y = +1 \\
			g(\bX)^2 &\text{if } Y = -1
		\end{cases} \\
		&= \big(\textbf{I}\{Y=+1\} - g(\bX)\big)^2.
\end{align*}
Thus, the limiting loss is minimized by the prediction $g(\bX) = \mathbb{E}_{Y|\bX}\{ \textbf{I}\{Y=+1\}\} = p(\bX)$, corresponding to the conditional class probability estimation task of soft classification.
%

\section{Proof of Theorem~\ref{thm:bayesrule}}
Let $\boldsymbol{\pi} = \{\pi_1,\ldots,\pi_K\}$ for some $K\geq1$ such that $0<\pi_1<\ldots<\pi_K<1$. Furthermore, let $h \in \{0,\ldots,K\}$ denote the index for some predicted $\omega_h \in \intervals$. Then, 
\begin{align*}
	\mathbb{E}_{Y|\bX} \big\{\ebpi^Y (\omega_h) \big\} 
		&= p(\bX) \cdot \ebpi^+(\omega_h) 
			+ (1-p(\bX)) \cdot \ebpi^-(\omega_h). \\
		&\propto p(\bX) \sum_{k=h+1}^{K+1} (1-\pi_k) 
			+ (1-p(\bX))\sum_{k=1}^{h} \pi_k,
\end{align*}
Letting $\pi_0=0$, $\pi_{K+1}=1$, we can express the above as:
\begin{align*}
	\mathbb{E}_{Y|\bX} \big\{\ebpi^Y (\omega_h) \big\} 
		&= \sum_{k=0}^{K+1} \Big\{ p(\bX)(1-\pi_k)\cdot \textbf{I}_{\{k>h\}} +
			\pi_k(1-p(\bX))\cdot \textbf{I}_{\{k \leq h\}} \Big\}.
\end{align*}
The sum is minimized by choosing $h$ such that $p(\bX)(1-\pi_k) \geq \pi_k(1-p(\bX))$ for all $k\leq h$ and $p(\bX)(1-\pi_k) \leq \pi_k(1-p(\bX))$ for all $k>h$. Thus, the optimal solution is given by $h^* = \argmax{k}\{ \pi_k < p(\bX) \}$. The equivalence between $\omega_{h^*}$ and $\sum_{k=0}^K \omega_k \cdot \textbf{I}\{p(\bX) \in \omega_k\}$ is immediate from the fact that $p(\bX) \in (\pi_{h^*}, \pi_{h^*+1}] = \omega_{h^*}$, and the additional assumption that $p(\bX) \not= \pi_k$ a.s. for all $k$.

\section{Proof of Theorem~\ref{thm:consistency}}
Let $\boldpi$ and $\bolddelta$ be appropriately defined boundaries in $(0, 1)$ and $\bbR$. Note that surrogate losses, $\surrplus$, $\surrminus$ are consistent for boundaries at $\boldpi$ with $\bolddelta$, i.e. $\boldpi,\bolddelta$-consistent, if and only if they are $\pi_k,\delta_k$-consistent for each $k$ separately. Thus, conditions for $\boldpi,\bolddelta$-consistency are simply the union of the conditions for $\pi_k,\delta_k$-consistency. Necessary and sufficient conditions for $\surrogate^Y$ to be $\pi_k,\delta_k$-consistent were provided by Theorem~1 of [12].

\section{Proof of Theorem~\ref{thm:piecewise}}
Let $\boldpi$ be an appropriately defined set of boundaries in $(0,1)$. Assume $\varsurrogate^Y$ to be defined as in (\ref{eq:varsurrplus}) and (\ref{eq:varsurrminus}) such that (C1)--(C3) are satisfied. We wish to show that for all $\pi_k \in \boldpi$, there exists some $\delta_k$ such that (\ref{eq:thmconsistency}) is satisfied, and furthermore, that there does not exist any $\delta$ such that (\ref{eq:thmconsistency}) is satisfied for $\pi\in(0,1)\setminus\boldpi$. Equivalently, we wish to show that $\varsurrogate^{-\prime}(x)/(\varsurrogate^{-\prime}(x)+\varsurrogate^{+\prime}(x))$ only takes values in $\boldpi$ over the set of $x$ such that $\varsurrogate^{-\prime}(x)<0$ and $\varsurrogate^{+\prime}(x)<0$ are defined. Note that $\varsurrogate^{+\prime}$ and $\varsurrogate^{-\prime}$ are only undefined at the hinge points, $H^Y(\pi_k,\pi_{k+1})$, $A^-(\pi_1)/B^-(\pi_1)$, and $A^+(\pi_K)/B^+(\pi_K)$. By (C2), the set of possible $\varsurrogate^{+\prime}$, $\varsurrogate^{-\prime}$ pairs are given by:
\begin{center}
\begin{tabular}{c r r r r r}
	\hline
	$\varsurrogate^{+\prime}$: & $B^+(\pi_1)$ & $B^+(\pi_1)$ & $\cdots$ & $B^+(\pi_K)$ & $0$ \\
	\hline
	$\varsurrogate^{-\prime}$: & $0$ & $B^-(\pi_1)$ & $\cdots$ & $B^-(\pi_K)$ & $B^-(\pi_K)$ \\
	\hline
\end{tabular}
\end{center}
Excluding the cases when $\varsurrogate^{+\prime}(x)=0$ or $\varsurrogate^{-\prime}(x)=0$, the set of possible consistent boundaries satisfying (\ref{eq:thmconsistency}) are given by:
\begin{align*}
	\frac{\varsurrogate^{-\prime}(x)}{\varsurrogate^{-\prime}(x)+\varsurrogate^{+\prime}(x)} 
		&= \frac{B^{-\prime}(\pi_k)}{B^{-\prime}(\pi_k) + B^{+\prime}(\pi_k)} = \pi_k \ \ \ \ \ \text{for } k=1,\ldots,K,
\end{align*}
where the final equality is given by (C3).

\section{Proof of Proposition~\ref{prop:logisticpiece}}
Let $\boldpi$ be an appropriately defined set of boundaries in $(0,1)$. We wish to show that (C1)--(C3) of Theorem~\ref{thm:piecewise} are satisfied for $A^+(\pi) = A^-(1-\pi) = -\pi\log(\pi) - (1-\pi)\log(1-\pi)$, and $B^+(\pi) = B^-(1-\pi) = -(1-\pi)$.

Trivially, (C1) is satisfied, as $B^+(\pi) = \pi-1$ and $B^-(\pi) = -\pi$ are non-decreasing and non-increasing, respectively, in $\pi$. To show that (C2) is satisfied, we derive the hinge points for the positive and negative class losses:
\begin{align*}
	H^+(\pi, \pi^\prime) 
		&= \frac{A^+(\pi) - A^+(\pi^\prime)}{B^+(\pi^\prime) - B^+(\pi)}
		= \frac{A^+(\pi) - A^+(\pi^\prime)}{\pi^\prime - \pi} \\
	H^-(\pi, \pi^\prime)
		&= \frac{A^-(\pi) - A^-(\pi^\prime)}{B^-(\pi^\prime) - B^-(\pi)}
		= -\frac{A^+(\pi) - A^+(\pi^\prime)}{\pi^\prime - \pi},
\end{align*}
where the final equality is obtained by noting $A^+(\pi) = A^+(1-\pi)$. The first equality of (C2) is clearly satisfied by the above derivations. We next show that the remaining three inequalities of (C2) are also satisfied. Let $k\in \{2,\ldots, K-1\}$. By the concavity of $A^+(\pi)$:
\begin{align*}
	H^+(\pi_{k-1}, \pi_{k})
		&= \frac{A^+(\pi_{k-1}) - A^+(\pi_{k})}{\pi_{k} - \pi_{k-1}} \\
		&= - \frac{A^+(\pi_{k}) - A^+(\pi_{k-1})}{\pi_{k} - \pi_{k-1}} \\
		&< -  (A^+)^\prime (\pi_k) \\
		&< - \frac{A^+(\pi_{k+1}) - A^+(\pi_{k})}{\pi_{k+1} - \pi_{k}} = H^+(\pi_{k}, \pi_{k+1}),
\end{align*}
Similarly, by the convexity of $A^-(\pi)$ and the fact that $\lim_{\pi\rightarrow0} A^+(\pi) = \lim_{\pi\rightarrow 1} A^+(\pi) = 0$, we have: 
\begin{align*}
	\frac{A^-(\pi_1)}{B^-(\pi_1)}
		&= -\frac{A^-(\pi_1) - \lim_{\pi\rightarrow0} A^-(\pi)}{\pi_1 - 0} \\
		&< -(A^-)^\prime (\pi_1) \\
		&< -\frac{A^-(\pi_{1}) - A^-(\pi_{2})}{\pi_{1} - \pi_{2}}
		= H^+(\pi_1, \pi_2) \\
	\frac{A^+(\pi_K)}{B^+(\pi_K)}
		&= -\frac{A^+(\pi_K) - \lim_{\pi\rightarrow1} A^+(\pi)}{\pi_K-1} \\
		&> -(A^+)^\prime (\pi_K) \\
		&> -\frac{A^+(\pi_{K-1}) - A^+(\pi_{K})}{\pi_{K-1} - \pi_{K}}
		= H^-(\pi_{K-1}, \pi_{K}).
\end{align*}
Thus (C2) is satisfied. Finally, (C3) holds, since for any $k = 1,\ldots, K$:
\begin{align*}
	\frac{B^-(\pi_k)}{B^-(\pi_k)+B^+(\pi_k)} 
		&= \frac{-\pi_k}{-\pi_k - (1-\pi_k)}
		= \pi_k.
\end{align*}

\section{Proof of Theorem~\ref{thm:bound}}
\label{A:proof-thm-bound}
Let $\phi^Y$ be a consistent surrogate loss for appropriately defined boundaries $\boldpi$ in $(0,1)$ at $\bolddelta$. First, note that the excess condition $\phi$-risk for a rule $g \in \mathcal{G}$ may be written as:
\begin{align*}
	R_p(g(\bx))
		&= \frac{2}{K} \Big[ (1-p(\bx))\sum_k \pi_k \textbf{I}\{g(\bx)>\pi_k\} +
					p(\bx)\sum_k (1-\pi_k) \textbf{I}\{g(\bx)\leq\pi_k\} \Big].
\end{align*}
Consider a candidate rule $g\in\mathcal{G}$, and recall the Bayes optimal rule over $\mathcal{G}$, $\bayesint(\bx)$, defined in Theorem~\ref{thm:bayesrule}. Suppose that $\bx \in \calX$ is such that $g(\bx) > \bayesint(\bx)$. Then, letting 
\begin{align*}
	\mathcal{K} =
	\begin{cases} 
		\{ k : g(\bX) \leq \pi_k < \bayesint(\bX) \}  &\text{if } \bayesint(\bX) > g(\bX) \\
		\{ k : \bayesint(\bX) \leq \pi_k < g(\bX) \}  &\text{if } \bayesint(\bX) < g(\bX) \\
		\emptyset &\text{otherwise}
	\end{cases},
\end{align*}
the excess condition $\phi$-risk may be expressed as:
\begin{align*}					
	\Delta R_p(g)
		&= \frac{2}{K} \Big[ (1-p(\bx))\sum_k \pi_k \textbf{I}\{\pi_k : \bayesint(\bx) \leq \pi_k < g(\bx) \} \\
		&\ \ \ \ \ \ \ \ - p(\bx)\sum_k (1-\pi_k) \textbf{I}\{\pi_k : \bayesint(\bx) \leq \pi_k < g(\bx) \} \Big] \\
		&= \frac{2}{K} \sum_{\mathcal{K}} \big[ (1-p(\bx))\pi_k - p(\bx)(1-\pi_k) \big] \\
		&= \frac{2}{K} \sum_{\mathcal{K}} \big[ \pi_k - p(\bx) \big].
\end{align*}
Similarly, for $g(\bx) < \bayesint(\bx)$, $\Delta R_p(g) = \tfrac{2}{K} \sum_{\mathcal{K}} \big[ p(\bx) - \pi_k \big]$. If $g(\bx) = \bayesint(\bx)$, we have that $\Delta R_p(g) = 0$, such that:
\begin{align*}
	\Delta R_p(g)
		&= \tfrac{2}{K} \sum_{\mathcal{K}} \big| p(\bx) - \pi_k \big|,
\end{align*}
for all $\bx\in\calX$.

By the stated assumptions, for $g(\bx) = C(f(\bx);\bolddelta)\in\mathcal{G}$, we immediately have the following result:
\begin{align*}
	(\Delta R_p(g))^s &=  \Big(\frac{2}{K} \sum_{\mathcal{K}} | p(\bX) - \pi_k | \Big)^s \\
		&\leq \frac{2}{K} \sum_{\mathcal{K}} |p(\bX)-\pi_k|^s \\
		&\leq  \frac{2}{K} C^s \sum_{\mathcal{K}} \Delta Q_p(\delta_k) \\
	\Delta R_p(g) &\leq C \Big(\frac{2}{K} \sum_{\mathcal{K}} \Delta Q_p(\delta_k) \Big)^{1/s}.
\end{align*}
Since $\Delta Q_p\geq0$, it suffices to show that $\sum_{\mathcal{K}} \Delta Q_p(\delta_k) \leq K\cdot\Delta Q_p(f)$. Since $|\mathcal{K}| \leq K$, we complete the proof by showing $\Delta Q_p(f) \geq \Delta Q_p(\delta_k)$ for all $k \in \mathcal{K}$. Without loss of generality, suppose $\bx$ is such that $g(\bx) < \bayesint(\bx)$ and let $k\in\mathcal{K}$. Note that $\pi_k < g(\bx)$ is equivalent to $\delta_k < f(\bx)$. By this fact and the convexity and consistency of $\surrogate^Y$, the following inequalities hold:
\begin{align*}
	\frac{\surrplus(f(\bx)) - \surrplus(\delta_k)}{f(\bx)-\delta_k} \geq \phi^{+\prime}(\delta_k) &&
	\frac{\surrminus(-f(\bx)) - \surrminus(-\delta_k)}{-f(\bx)+\delta_k} \leq \phi^{-\prime}(-\delta_k).
\end{align*}
Thus,
\begin{align*}
	Q_p(f) - Q_p(\delta_k) 
		&= p(\bx)(\surrplus(f(\bx)) - \surrplus(\delta_k)) + (1-p(\bx))(\surrminus(-f(\bx)) - \surrminus(-\delta_k)) \\
		&\geq p(\bx)(f(\bx)-\delta_k)\phi^{+\prime}(\delta_k) - (1-p(\bx))(f(\bx)-\delta_k)\phi^{-\prime}(-\delta_k) \\
		&\geq (f(\bx)-\delta_k)\big\{p(\bx)\big(\phi^{+\prime}(\delta_k) + \phi^{-\prime}(-\delta_k)\big) 
			- \phi^{-\prime}(-\delta_k)\big\} \\
		&\geq (f(\bx)-\delta_k)\big\{p(\bx)\tfrac{\phi^{-\prime}(-\delta_k)}{\pi_k} - \phi^{-\prime}(-\delta_k)\big\} \\
		&\geq (f(\bx)-\delta_k)\ \phi^{-\prime}(-\delta_k)\ (\tfrac{p(\bx)}{\pi_k}-1).
\end{align*}
Since $f(\bx)-\delta_k >0$, $\phi^{-\prime}(-\delta_k)<0$, and $p(\bx)<\pi_k$, $Q_p(f) - Q_p(\delta_k) \geq 0$. The case when $g(\bx) < \bayesint(\bx)$ follows similarly, and the proof is complete.

\section{Proof of Theorem~\ref{thm:bound2}}
Let $\phi^Y$ be a consistent surrogate loss for appropriately defined boundaries $\boldpi$ in $(0,1)$ at $\bolddelta$. Throughout, we use $g = C(f; \bolddelta)$ to denote the corresponding rule in $\mathcal{G}$ for some margin function $f \in \mathcal{F}$. From the proof of Theorem~\ref{thm:bound}, we have that:
\begin{align*}
	\Delta R(g) &= \frac{2}{K} \cdot \mathbb{E} \Big\{ \sum_{\mathcal{K}} |p(\bX) - \pi_k| \Big\} \\
		&= \frac{2}{K} \cdot \mathbb{E} \Big\{ \sum_{k=1}^K |p(\bX) - \pi_k| \cdot \bold{I}\{k\in\mathcal{K}\} \Big\} \\
		&= \frac{2}{K} \cdot \sum_{k=1}^K \mathbb{E} \Big\{ |p(\bX) - \pi_k| \cdot \bold{I}\{k\in\mathcal{K}\} \Big\},
\end{align*}
where $\mathcal{K}$ is defined as in the proof of Theorem~\ref{thm:bound} (Section~\ref{A:proof-thm-bound}). Additionally, note that for fixed $k \in \{1,\ldots, K\}$:
\begin{align*}
	\mathbb{E} \Big\{ |p(\bX) - \pi_k| \cdot \bold{I}\{k\in\mathcal{K}\} \Big\}
		&\geq t \cdot \mathbb{P} \big\{ (k\in\mathcal{K}) \cap |p(\bX) - \pi_k| > t \big\} \\
		&= t \cdot \mathbb{P} \big\{ |p(\bX) - \pi_k| > t \big\} 
			- t \cdot \mathbb{P} \big\{ (k\not\in\mathcal{K}) \cap |p(\bX) - \pi_k|>t \big\} \\
		&\geq t \cdot (1- At^\alpha) - t \cdot \mathbb{P}\{ k \not\in\mathcal{K} \} \\
		&= t \cdot \big( \mathbb{P}\{k\in\mathcal{K}\} - At^\alpha \big).
\end{align*}
Combining the above inequalities, we have:
\begin{align*}
	\Delta R(g) 
		&\geq \frac{2t}{K}\cdot \Big( \sum_{k=1}^K \mathbb{P}\{k\in\mathcal{K}\} - KAt^\alpha \Big) \\
		&\geq \frac{2t}{K}\cdot \Big( \mathbb{P}\{f \not= f^*\} - KAt^\alpha \Big).
\end{align*}
Letting $t = (\tfrac{\mathbb{P}\{f\not=f^*\}}{2KA})^{1/\alpha}$ and using $\beta$ to denote $\alpha/(1+\alpha)$,
\begin{align*}
	\Delta R(g)
		&\geq \frac{2}{K}\cdot\Big( \frac{\mathbb{P}\{f\not=f^*\}}{2KA} \Big)^{1/\alpha} 
			\cdot \Big( \frac{\mathbb{P}\{f\not=f^*\}}{2} \Big) \\
		&= \frac{\mathbb{P}\{f\not=f^*\}^{(1+\alpha)/\alpha}}{(2A)^{1/\alpha}K^{(1+\alpha)/\alpha} } \\
	\frac{\mathbb{P}\{f\not=f^*\}}{K}
		&\leq \big( (2A)^{1/\alpha} \Delta R(g) \big)^\beta.
\end{align*}
Now consider,
\begin{align*}
	\Delta R(g) 
		&= \frac{2}{K}\cdot \sum_{k=1}^K \mathbb{E} \big( |p(\bX) - \pi_k| \cdot \bold{I}\{k\in\mathcal{K}\} \big) \\
		&= \frac{2}{K}\cdot \sum_{k=1}^K \mathbb{E} \big( |p(\bX) - \pi_k| \cdot \bold{I}\{k\in\mathcal{K}\} 
				\cdot \bold{I}\{|p(\bX) - \pi_k| > \epsilon\} \big) \\
			&\ \ \ + \frac{2}{K}\cdot\sum_{k=1}^K \mathbb{E} \big( |p(\bX) - \pi_k| \cdot \bold{I}\{k\in\mathcal{K}\} 
				\cdot \bold{I}\{|p(\bX) - \pi_k| \leq \epsilon\} \big).
\end{align*}
Using the inequality: $|x|\cdot\bold{I}\{|x|\geq\epsilon\} \leq |x|^s \cdot \epsilon^{1-s}$ for $s\geq 1$, we have:
\begin{align*}
	\Delta R(g)
		&\leq \frac{2}{K}\cdot \sum_{k=1}^K \mathbb{E} \big( |p(\bX) - \pi_k|^s \cdot \epsilon^{1-s} \cdot %
				\bold{I}\{k\in\mathcal{K}\} \big) \\
		&\ \ \ + \frac{2\epsilon}{K}\cdot \sum_{k=1}^K \mathbb{P} \{k\in\mathcal{K}\}.
\end{align*}
From the proof of Thoerem~\ref{thm:bound} (Section~\ref{A:proof-thm-bound}), $|p(\bX) - \pi_k|^s \leq C^s \Delta Q_p(\delta_k) \leq C^s \Delta Q_p(f)$ for $k\in\mathcal{K}$. Therefore,
\begin{align*}
	\Delta R(g)
		&\leq 2\epsilon^{1-s} C^s \Delta Q(f) + \frac{2\epsilon}{K} \cdot \mathbb{P}\{f\not=f^*\}.
\end{align*}
Combining with the previous bound on $\mathbb{P}\{f\not=f^*\}$,
\begin{align*}
	\Delta R(g)
		&\leq 2\epsilon^{1-s} C^s \Delta Q(f) + 2\epsilon \cdot \big((2A)^{1/\alpha} \Delta R(g) \big)^\beta
\end{align*}
Further choosing $\epsilon = \Delta R(g)^{1-\beta}$,
\begin{align*}
	\Delta R(g)
		&\leq 2\Delta R(g)^{(1-\beta)(1-s)} C^s \Delta Q(f) + 2 (2A)^{1/\alpha} \Delta R(g) \\
	(1-2(2A)^{1/\alpha}) \Delta R(g)^{s+\beta-s\beta}
		&\leq 2 C^s \Delta Q(f) \\
	\Delta R(g)
		&\leq \Big( \frac{2 C^s}{1-2(2A)^{1/\alpha}} \Big)^{1/(s+\beta-s\beta)} \cdot \Delta Q(f)^{1/(s+\beta-s\beta)}
\end{align*}
Letting $D$ denote the exponentiated fraction on the right of the inequality,
\begin{align*}
	\Delta R(g) &\leq D \cdot \Delta Q(f)^{1/(s+\beta-s\beta)}.
\end{align*}
%

\section{Proof of Corollary~\ref{thm:PLbounds}}
Let $\varsurrogate^Y$ be a minimally consistent piecewise linear surrogate loss for appropriately defined boundaries $\boldpi$ in $(0,1)$ at $\bolddelta$. The $\varsurrogate^Y$-optimal margin function, denoted by $f^*_{\varsurrogate}$, is given by:
\begin{align*}
	f_{\varsurrogate}^* (\bX)
		&= \argmin{f} \mathbb{E}_{Y|\bX} \{\varsurrogate^Y(Yf(\bX)) \} \\
		&= \argmin{f} \big\{p(\bX) \varsurrplus(f(\bX)) + (1-p(\bX)) \varsurrminus(-f(\bX)) \big\} \\
		&=
		\begin{cases}
			A^-(\pi_1) / B^-(\pi_1) &\text{if } p(\bX) \in [0, \pi_1) \\
			H^+(\pi_1, \pi_2) &\text{if } p(\bX) \in  (\pi_1, \pi_2] \\
			\cdots \\
			H^+(\pi_{K-1}, \pi_{K}) &\text{if } p(\bX) \in (\pi_{K-1}, \pi_K] \\
			A^+(\pi_K) / B^+(\pi_{K}) &\text{if } p(\bX) \in (\pi_K, 1]
		\end{cases}.
\end{align*}
For any $k \in \{1,\ldots,K\}$,
\begin{align*}
	\Delta Q_p(\delta_k)
		&= Q_p(\delta_k) - Q_p(f_{\varsurrogate}^*(\bX)) \\
		&= p(\bX)(\varsurrplus(\delta_k) - \varsurrplus(f_{\varsurrogate}^*(\bX))) 
			+ (1-p(\bX))(\varsurrminus(-\delta_k) - \varsurrminus(-f_{\varsurrogate}^*(\bX))) \\
		&= p(\bX) B^+(\pi_k) (\delta_k - f_{\varsurrogate}^*(\bX)) 
			- (1-p(\bX)) B^-(\pi_k) (\delta_k - f_{\varsurrogate}^*(\bX)) \\
		&= p(\bX) (B^+(\pi_k) + B^-(\pi_k)) (\delta_k - f_{\varsurrogate}^*(\bX))
			- B^-(\pi_k) (\delta_k - f_{\varsurrogate}^*(\bX)) \\
		&= p(\bX) (B^-(\pi_k)\cdot \pi_k^{-1}) (\delta_k - f_{\varsurrogate}^*(\bX))
			- B^-(\pi_k) (\delta_k - f_{\varsurrogate}^*(\bX)) \\
		&= B^-(\pi_k) \cdot \pi_k^{-1} \cdot (p(\bX) - \pi_k) (\delta_k - f_{\varsurrogate}^*(\bX)).
\end{align*}
Since $f_{\varsurrogate}^*(\bX) > \delta_k$ when $p(\bX) > \pi_k$, and similarly $f_{\varsurrogate}^*(\bX) < \delta_k$ when $p(\bX) < \pi_k$, $(p(\bX) - \pi_k)(\delta_k - f_{\varsurrogate}^*(\bX)) \leq 0$ must always hold. Therefore,
\begin{align*}
	\Delta Q_p(\delta_k)
		&= - \frac{B^-(\pi_k) \cdot |\delta_k - f_{\varsurrogate}^*(\bX)|}{\pi_k} \cdot |p(\bX) - \pi_k| \\
		&\geq C^{-1} \cdot |p(\bX) - \pi_k| ,
\end{align*}
where $C = \max \left\{ -\frac{\pi_k}{B^-(\pi_k) \cdot |\delta_k - H_j|} : k=1,\ldots,K;\ j=0,\ldots,K\right\} > 0$. Letting $s = 1$, the desired bound is achieved.
%

\section{Proof of Theorem~\ref{thm:piececonv}}
Let $\varsurrogate^Y$ be a minimally consistent piecewise linear surrogate loss for appropriately defined boundaries $\boldpi$ in $(0,1)$ at $\bolddelta$. We first show that $\mathcal{H} = \{h_f (\bx, y) = \varsurrogate^y(yf(\bx)) - \varsurrogate^y(yf_{\varsurrogate}^*(\bx)) : f\in\mathcal{F}\}$ is a Bernstein class of functions, i.e. that there exists some $B>1$, $\beta\in(0,1]$ such that:
\begin{align*}
 	\mathbb{E} \{h_f (\bX, Y)^2\} \leq B\cdot \mathbb{E} \{h_f (\bX, Y)\}^\beta.
\end{align*}
Then, given that $h_f$ is a Bernstein class, we complete the proof by obtaining a tail bound on $\mathbb{E} h_f (\bX, Y) - 2 \tfrac{1}{n} \sum_i h_f (\bx_i, y_i)$. Following the approach of [11], to derive the Bernstein property of $h_f$, we first show that $\Delta Q_p(f)$ can be bounded below by a pseudo-norm between $f$ and $f_{\varsurrogate}^*$, denoted $\rho_{\bX}(f,f_{\varsurrogate}^*)$. Then, we show that $\mathbb{E}\{h_f (\bX, Y)^2\}$ can bounded above by $\mathbb{E}\{\rho_{\bX}(f,f_{\varsurrogate}^*)\}$, and combine the two results to show the Bernstein property of $h_f$. Let $\rho_{\bX}(f, f_{\varsurrogate}^*)$ be defined as:
\begin{align*}
	\rho_{\bX}(f, f_{\varsurrogate}^*) = 
	\begin{cases}
		p(\bX) |f - f_{\varsurrogate}^*| &\text{if } p(\bX) < \pi_1,\ f<H_0 \\
		(1-p(\bX)) |f - f_{\varsurrogate}^*| &\text{if } p(\bX) > \pi_K,\ f>H_{K} \\
		|f - f_{\varsurrogate}^*| &\text{otherwise}
	\end{cases}.
\end{align*}
%

%
\begin{lemma} \label{lemma:rhoL1}
For $p(\bX)\in [0,1]$,
	\[
		\Delta Q_p(f) 
			\geq D^* \cdot \min\{ |p(\bX)-\pi_1|, |p(\bX) - \pi_K|, (1-\pi_1), \pi_K \} 
				\cdot \rho_{\bX}(f,f_{\varsurrogate}^*).
	\]
\end{lemma}

\begin{proof}
Since $Q_p(f)$ is convex, $Q_p(f) \geq Q_p(f_{\varsurrogate}^*) + r \cdot (f - f_{\varsurrogate}^*)$ for any subgradient, $r$, of $Q_p(\cdot)$ at $f_{\varsurrogate}^*$. Since $\varsurrogate^Y$ is piecewise linear, and $f_{\varsurrogate}^*$ is as defined above, the set of subgradients are given by:
\begin{align*}
	r = 
	\begin{cases}
		\!\begin{aligned}[c]
			& p(\bX) B^+(\pi_1) \\
			& \ \ \ \ \text{ and }\ \ p(\bX)B^+(\pi_1) + (1-p(\bX))B^-(\pi_1) 
		\end{aligned}
			&\text{for } f^*_{\varsurrogate} = H_0 \\
		\!\begin{aligned}[c]
			& p(\bX)B^+(\pi_1) + (1-p(\bX))B^-(\pi_1) \\
			& \ \ \ \ \text{ and }\ \ p(\bX)B^+(\pi_1) + (1-p(\bX))B^-(\pi_1) 
		\end{aligned}
			&\text{for } f^*_{\varsurrogate} = H_1,\ldots, H_{K-1} \\
		\!\begin{aligned}[c]
			& (1-p(\bX)) B^-(\pi_K) \\
			& \ \ \ \ \text{ and }\ \ p(\bX)B^+(\pi_K) + (1-p(\bX))B^-(\pi_K) 
		\end{aligned}
			&\text{for } f^*_{\varsurrogate} = H_{K}
	\end{cases}.
\end{align*}
Therefore, 
\begin{align*}
	Q_p(f)
		&\geq Q_p(f_{\varsurrogate}^*) + r \cdot (f - f_{\varsurrogate}^*) \\
	\Delta Q_p(f) 
		&\geq r \cdot (f - f_{\varsurrogate}^*) \\
		&\geq
		\begin{cases}
			\big(p(\bX) B^+(\pi_1)\big) \cdot \big(f - f_{\varsurrogate}^* \big)
				&\text{if } p(\bX) < \pi_1,\ f < H_0 \\
			\big(p(\bX) B^+(\pi_1) - (1-p(\bX))B^-(\pi_1)\big) \cdot \big(f - f_{\varsurrogate}^* \big)
				&\text{if } p(\bX) < \pi_1,\ f > H_0 \\
			\big(p(\bX) B^+(\pi_{k}) - (1-p(\bX))B^-(\pi_{k})\big) \cdot \big(f - f_{\varsurrogate}^* \big)
				&\text{if } p(\bX) \in [\pi_k, \pi_{k+1}),\ f < H_k \\
			\big(p(\bX) B^+(\pi_{k+1}) - (1-p(\bX))B^-(\pi_{k+1})\big) \cdot \big(f - f_{\varsurrogate}^* \big)
				&\text{if } p(\bX) \in [\pi_k, \pi_{k+1}),\ f > H_k \\
			\big(p(\bX) B^+(\pi_K) - (1-p(\bX))B^-(\pi_K)\big) \cdot \big(f - f_{\varsurrogate}^* \big)
				&\text{if } p(\bX) > \pi_K,\ f < H_K \\
			(1-p(\bX)) B^-(\pi_K) \cdot \big(f - f_{\varsurrogate}^* \big)
				&\text{if } p(\bX) > \pi_K,\ f > H_K
		\end{cases}.
\end{align*}
Since by definition, $B^+(\pi_1) \leq B^+(\pi_2) \leq \cdots \leq B^+(\pi_K)$ and $B^-(\pi_1) \geq B^-(\pi_2) \geq \cdots \geq B^-(\pi_K)$, we have:
\begin{align*}
	-p(\bX) B^+(\pi_{k}) + (1-p(\bX)) B^-(\pi_{k}) 
		&\geq -p(\bX) B^+(\pi_{K}) + (1-p(\bX)) B^-(\pi_{K}) \\
	p(\bX) B^+ (\pi_{k+1}) - (1-p(\bX)) B^-(\pi_{k+1})
		&\leq p(\bX) B^+ (\pi_{1}) - (1-p(\bX)) B^-(\pi_{1}).
\end{align*}
Therefore, the bound on $\Delta Q_p(f)$ may be rewritten as:
\begin{align*}
	\Delta Q_p(f) 
		&\geq
		\begin{cases}
			\big|p(\bX) B^+(\pi_1)\big| \cdot \big|f - f_{\varsurrogate}^* \big|
				&\text{if } p(\bX) < \pi_1,\ f < H_0 \\
			\big|p(\bX) B^+(\pi_1) - (1-p(\bX))B^-(\pi_1)\big| \cdot \big|f - f_{\varsurrogate}^* \big|
				&\text{if } p(\bX) < \pi_1,\ f > H_0 \\
			\big|p(\bX) B^+(\pi_K) - (1-p(\bX))B^-(\pi_K)\big| \cdot \big|f - f_{\varsurrogate}^* \big|
				&\text{if } p(\bX) > \pi_K,\ f < H_K \\
			\big|(1-p(\bX)) B^-(\pi_K)\big| \cdot \big|f - f_{\varsurrogate}^* \big|
				&\text{if } p(\bX) > \pi_K,\ f > H_K \\
			\!\begin{aligned}[c]
				& \min\big\{ \big|p(\bX) B^+(\pi_{1}) - (1-p(\bX))B^-(\pi_{1})\big|, \\
				& \ \ \ \ \big|p(\bX) B^+(\pi_{K}) - (1-p(\bX))B^-(\pi_{K})\big| \big\}
				\cdot \big|f - f_{\varsurrogate}^* \big|
			\end{aligned}
				&\text{otherwise}
		\end{cases}.
\end{align*}
By the consistency of $\varsurrogate^Y$, $B^-(\pi_k)/(B^+(\pi_k) + B^-(\pi_k)) = \pi_k$ for all $k$. Thus, letting $D^* = \min_{k=1,\ldots,K} \{ | B^+(\pi_k) + B^-(\pi_k)| \} >0$, $p(\bX) B^+(\pi_k) - (1-p(\bX))B^-(\pi_k) = (p(\bX)-\pi_k) (B^+(\pi_k) + B^-(\pi_k)) \geq D^* \cdot |p(\bX)-\pi_k|$. Therefore, 
\begin{align*}
	\Delta Q_p(f) 
		&\geq
		\begin{cases}
			\big|B^+(\pi_1)\big| \cdot \rho_{\bX}(f, f_{\varsurrogate}^*)
				&\text{if } p(\bX) < \pi_1,\ f < H_0 \\
			D^* \cdot |p(\bX)-\pi_1| \cdot \rho_{\bX}(f, f_{\varsurrogate}^*)
				&\text{if } p(\bX) < \pi_1,\ f > H_0 \\
			D^* \cdot |p(\bX)-\pi_K| \cdot \rho_{\bX}(f, f_{\varsurrogate}^*)
				&\text{if } p(\bX) > \pi_K,\ f < H_K \\
			\big|B^-(\pi_K)\big| \cdot \rho_{\bX}(f, f_{\varsurrogate}^*)
				&\text{if } p(\bX) > \pi_K,\ f > H_K \\
			\!\begin{aligned}[c]
				& D^* \cdot \min\big\{|p(\bX)-\pi_1|, |p(\bX)-\pi_K|\big\}
				\cdot \rho_{\bX}(f, f_{\varsurrogate}^*)
			\end{aligned}
				&\text{otherwise}
		\end{cases}.
\end{align*}
Since $|B^+(\pi_1)| \geq D^* \cdot (1-\pi_1)$, $|B^-(\pi_K)| \geq D^* \cdot \pi_K$, we have for $p(\bX)\in [0,1]$:
\begin{align*}
	\Delta Q_p(f) 
		\geq D^* \cdot \min\big\{ |p(\bX)-\pi_1|, |p(\bX) - \pi_K|, (1-\pi_1), \pi_K \big\} 
			\cdot \rho_{\bX}(f,f_{\varsurrogate}^*).
\end{align*}
\end{proof}

%
\begin{lemma} \label{lemma:rhoL2}
If $|f|<B$ for all $f\in\mathcal{F}$, 
	\[ \mathbb{E}_{Y|\bX} \{ h_f(\bX, Y)^2 \} \leq L^2 (B+M) \cdot \rho_{\bX}(f, f_{\varsurrogate}^*) \]
for $L, M \geq 0$.
\end{lemma}

\begin{proof}
We first decompose the conditional expectation as:
\begin{align*}
	\mathbb{E}_{Y|\bX} \{h_f (\bX, Y)^2\} 
		&= \mathbb{E}_{Y|\bX} \left\{ \big(\varsurrogate^Y(Yf(\bX)) - \varsurrogate^Y(Yf(\bX))\big)^2 \right\} \\
		&= p(\bX) \big( \varsurrplus(f(\bX)) - \varsurrplus(f_{\varsurrogate}^*(\bX))\big)^2 \\
			&\ \ \ + (1-p(\bX)) \big( \varsurrminus(f(\bX)) - \varsurrminus(f_{\varsurrogate}^*(\bX))\big)^2.
\end{align*}
Note that if $f(\bX)\leq H_0$ and $p(\bX) \leq \pi_1$, then $\varsurrminus(f(\bX))=0$ and $\varsurrminus(f^*_{\varsurrogate}(\bX))=0$. Similarly, if $f(\bX)\geq H_K$ and $p(\bX) \geq \pi_K$, then $\varsurrplus(f(\bX))=0$ and $\varsurrplus(f^*_{\varsurrogate}(\bX))=0$. Therefore,
\begin{align*}
	E_{Y|\bX} \{h_f (\bX, Y)^2\} 
		&= 
		\begin{cases}
			p(\bX) \big( \varsurrplus(f(\bX)) - \varsurrplus(f_{\varsurrogate}^*(\bX))\big)^2 
				& \text{if } f(\bX) \leq H_0,\ p(\bX) \leq \pi_1 \\
			(1-p(\bX)) \big( \varsurrminus(f(\bX)) - \varsurrminus(f_{\varsurrogate}^*(\bX))\big)^2 
				& \text{if } f(\bX) \geq H_K,\ p(\bX) \geq \pi_K \\
			\!\begin{aligned}[c]
				& p(\bX) \big( \varsurrplus(f(\bX)) - \varsurrplus(f_{\varsurrogate}^*(\bX))\big)^2 \\
				& \ \ \ + (1-p(\bX)) \big( \varsurrminus(f(\bX)) - \varsurrminus(f_{\varsurrogate}^*(\bX))\big)^2 
			\end{aligned}
				& \text{otherwise}
		\end{cases}.
\end{align*}
Let $L = \max\{B^+(\pi_1), B^-(\pi_K)\}$ denote the Lipschitz constant for $\varsurrogate^Y$, and let $M=\max\{|H_0|, |H_{K}|\}$ denote the bound on $f_{\varsurrogate}^*$, such that $|f_{\varsurrogate}^*(\bX)| \leq M$ for all $\bX$. Then, 
\begin{align*}
	E_{Y|\bX} \{h_f (\bX, Y)^2\}
		&\leq L^2 (B+M) \cdot \rho_{\bX}(f, f_{\varsurrogate}^*),
\end{align*}
where $\rho_{\bX}$ is as defined above.
\end{proof}

%
\begin{lemma} \label{lemma:Bernstein}
If $p(\bX)$ satisfies the margin condition (\ref{eq:margincondition}) at $\boldpi = \{\pi_1, \ldots, \pi_K\}$ with parameters $A,\alpha$, then for any class $\mathcal{F}$ of measurable uniformly bounded functions, the class $\mathcal{H} = \{h_f (\bX, Y) : f\in\mathcal{F}\}$ is a Bernstein class with exponent $\beta = \alpha/(1+\alpha)$.
\end{lemma}

\begin{proof}
Let $E_1$ denote the event that $|p(\bX)-\pi|$ is the minimizer over the set $\big\{ |p(\bX)-\pi_1|, |p(\bX) - \pi_K|, (1-\pi_1), \pi_K \big\}$, and let $E_2$, $E_3$, $E_4$ similarly denote the corresponding events for $|p(\bX)-\pi_K|$, $(1-\pi_1)$ and $\pi_K$. Using $\textbf{I}_{E}$ to denote the indicator for event $E$, by Lemma~\ref{lemma:rhoL2} we have: 
\begin{align*}
	\mathbb{E}\{h_f(\bX, Y)\} 
		&\geq D^* \cdot \mathbb{E} \big\{ \min \{ |p(\bX)-\pi_1|, |p(\bX) - \pi_K|, (1-\pi_1), \pi_K \} 
			\cdot \rho_{\bX}(f,f_{\varsurrogate}^*) \big\} \\
		&= D^* \cdot \mathbb{E} \big\{ \rho_{\bX} (f, f^*) \cdot
				\{ \textbf{I}_{E_1} \cdot |p(\bX)-\pi_1| 
				+ \textbf{I}_{E_2} \cdot |p(\bX)-\pi_K| \\
			&\ \ \ \ \ \ \ \ \ \ \ \ \ \ \ \ \ \ \ \ \ \ \ \ \ \ \ \ \ + \textbf{I}_{E_3} \cdot (1-\pi_1) 
				+\textbf{I}_{E_4} \cdot \pi_K \}\big\}.
\end{align*}
Let $t_{\max} = \min_{k=1,\ldots,K+1}\{\pi_k - \pi_{k-1}\}$, where $\pi_0=0, \pi_{K+1}=1$. Given the margin condition, for all $k$, there exists some $A\geq0, \alpha\geq0$ such that for all $t\in[0, t_{\max})$, 
	\[
		\mathbb{P}\{|p(\bX) - \pi_k| \leq t\} \leq At^\alpha,
	\]
for $k=1,\ldots, K$. Therefore, letting $B$ and $M$ denote the bounds on $f$ and $f_{\varsurrogate}^*$ given in the proof of Lemma~\ref{lemma:rhoL2},
\begin{align*}
	\mathbb{E} \big\{ \rho_{\bX} (f, f_{\varsurrogate}^*) \cdot |p(\bX) - \pi_1 | \cdot \textbf{I}_{E_1} \big\}
		&\geq t \cdot \mathbb{E} \big\{ \rho_{\bX} (f, f_{\varsurrogate}^*) \cdot 
			\textbf{I}\{|p(\bX) - \pi_1| > t\} \cdot \textbf{I}_{E_1} \big\} \\
		&\geq t \cdot \big[ \mathbb{E} \big\{ \rho_{\bX} (f, f_{\varsurrogate}^*) \cdot \textbf{I}_{E_1} \big\}
			- (B+M)\cdot At^\alpha \big],
		\intertext{and similarly,}
	\mathbb{E} \big\{ \rho_{\bX} (f, f_{\varsurrogate}^*) \cdot |p(\bX) - \pi_K| \cdot \textbf{I}_{E_2} \big\}
		&\geq t \cdot \big[ \mathbb{E} \big\{ \rho_{\bX} (f, f_{\varsurrogate}^*) \cdot \textbf{I}_{E_2} \big\}
			- (B+M)\cdot At^\alpha \big] \\
	\mathbb{E} \big\{ \rho_{\bX} (f, f_{\varsurrogate}^*) \cdot (1-\pi_1) \cdot \textbf{I}_{E_3} \big\}
		&\geq t \cdot \big[ \mathbb{E} \big\{ \rho_{\bX} (f, f_{\varsurrogate}^*) \cdot \textbf{I}_{E_3} \big\}
			- (B+M)\cdot \textbf{I}\{(1-\pi_1) < t, (1-\pi_1) \leq \pi_K\} \big] \\
	\mathbb{E} \big\{ \rho_{\bX} (f, f_{\varsurrogate}^*) \cdot \pi_K \cdot \textbf{I}_{E_4} \big\}
		&\geq t \cdot \big[ \mathbb{E} \big\{ \rho_{\bX} (f, f_{\varsurrogate}^*) \cdot \textbf{I}_{E_4} \big\}
			- (B+M)\cdot \textbf{I}\{\pi_K < t, \pi_K < (1-\pi_1)\} \big].
\end{align*}
Assume without loss of generality that $\pi_K < (1-\pi_1)$. Let
	\[
		t = \left( \frac{\mathbb{E} \big\{ \rho_{\bX} (f, f_{\varsurrogate}^*)\big\}}{C\cdot 2A(B+M)}\right)^{1/\alpha},
	\]
where $C \geq \max\{ 2, (2A\pi_K^\alpha)^{-1}\}$. Then, since $\mathbb{E}\{\rho_{\bX} (f, f_{\varsurrogate}^*)\} \leq (B+M)$, we have $t<\pi_K$. Combining the above inequalities, we have:
\begin{align*}
	\mathbb{E}\{h_f(\bX, Y)\}
		&\geq D^* \cdot t \cdot \big[ \mathbb{E}\big\{ \rho_{\bX}(f, f_{\varsurrogate}^*) \big\}
			- (B+M) (2At^\alpha)\big] \\
		&\geq D^* \cdot \left( \frac{\mathbb{E} 
			\big\{ \rho_{\bX} (f, f_{\varsurrogate}^*) \big\}}{C\cdot 2A(B+M)}\right)^{1/\alpha}
			\big[ \mathbb{E}\{\rho_{\bX}(f, f_{\varsurrogate}^*)\} - 
			C^{-1} \mathbb{E}\big\{\rho_{\bX}(f, f_{\varsurrogate}^*)\big\}\big] \\
		&\geq D^* \cdot \left(\frac{C-1}{C}\right) \cdot 
			\left(\frac{1}{C\cdot 2A(B+M)}\right)^{1/\alpha} \cdot \mathbb{E} 
			\big\{ \rho_{\bX}(f, f_{\varsurrogate}^*)\big\}^{(1+\alpha)/\alpha} \\
	\mathbb{E}\big\{ \rho_{\bX}(f, f_{\varsurrogate}^*) \big\}
		&\leq \left[ \left(\frac{C}{C-1}\right) \cdot D^* \cdot (C\cdot 2A(B+M))^{1/\alpha}\right]^{\alpha/(1+\alpha)}
			\cdot \mathbb{E}\{h_f(\bX, Y)\}^{\alpha/(1+\alpha)}.
\end{align*}
Combining with the result of Lemma~\ref{lemma:rhoL2}, and noting that $\mathbb{E}\{\rho_{\bX}(f,f_{\varsurrogate}^*)\} = \mathbb{E}_{\bX} \{\rho_{\bX}(f,f_{\varsurrogate}^*)\}$, we have:
\begin{align*}
	\mathbb{E} \big\{ h_f(\bX, Y)^2\big\} 
	 	&= \mathbb{E}_{\bX} \big\{ \mathbb{E}_{Y|\bX} \{ h_f(\bX, Y)^2 \} \\
		&\leq L^2 (B+M) \cdot \mathbb{E}_{\bX} \big\{ \rho_{\bX}(f, f_{\varsurrogate}^*) \big\} \\
		&\leq L^2 (B+M) \cdot \left[ \left(\frac{C}{C-1}\right) \cdot D^* \cdot (C\cdot 2A(B+M))^{1/\alpha}\right]^{\alpha/(1+\alpha)}
			\cdot \mathbb{E}\{h_f(\bX, Y)\}^{\alpha/(1+\alpha)},
\end{align*}
such that $h_f$ is a Bernstein class. 
\end{proof}

%
Let $B^\prime$ and $\beta$ be defined such that $\mathbb{E} \big\{ h_f(\bX, Y)^2\big\} \leq B^\prime \cdot \mathbb{E}\{h_f(\bX, Y)\}^\beta$. Let $\hat{f}_n$ denote the empirical minimizer in $\mathcal{F}$ of $\varsurrogate^y(yf(\bx))$ over a training sample of size $n$. We first bound the excess $\varsurrogate$-risk by:
\begin{align*}
	\Delta Q(\hat{f}_n)
		&= \mathbb{E} \{ h_{\hat{f}_n}(\bX, Y) \} \\
		&= 2 \Big(\frac{1}{n} \sum_{i=1}^n h_{\widehat{f}_n} (\bx_i, y_i) \Big) + 
			\Big(\mathbb{E} \{h_{\widehat{f}_n}(\bX, Y)\} - 
			2 \big(\frac{1}{n} \sum_{i=1}^n h_{\widehat{f}_n} (\bx_i, y_i) \big)\Big) \\
		&\leq \sup_{f\in\mathcal{F}_B} \Big( \mathbb{E} \{h_{f}(\bX, Y)\} - 
			2 \big(\frac{1}{n} \sum_{i=1}^n h_{f} (\bx_i, y_i) \big)\Big).
\end{align*}
Note that,
\begin{align*}
	\sup_{f\in\mathcal{F}_B} \Big( \mathbb{E} \{h_{f}(\bX, Y)\} - 
			2 \big(\frac{1}{n} \sum_{i=1}^n h_{f} (\bx_i, y_i) \big)\Big)
		&\leq \frac{3L}{n} + \sup_{f\in\mathcal{F}_n} \Big( \mathbb{E} \{h_{f}(\bX, Y)\} - 
			2 \big(\frac{1}{n} \sum_{i=1}^n h_{f} (\bx_i, y_i) \big)\Big),
\end{align*}
where $\mathcal{F}_n$ is a minimal $1/n$-net of $\mathcal{F}_B$. Now applying Bernstein's inequality,
\begin{align*}
	&\mathbb{P}\Big\{ \sup_{f\in\mathcal{F}_n} \Big( \mathbb{E} \{h_{f}(\bX, Y)\} - 
			2 \big(\frac{1}{n} \sum_{i=1}^n h_{f} (\bx_i, y_i) \big)\Big) \geq t\Big\} \\
		&\ \ \ \ \ \ \ \ \ \ 
		\leq N_n \cdot \text{exp} \Big\{ -
			\frac{n(t + \mathbb{E}\{h_f(\bX,Y)\})^2/8}%
			{\mathbb{E}\{h_f(\bX,Y)^2\} + (2LB)(t+\mathbb{E}\{h_f(\bX,Y)\})/6} 
			\Big\}.
\end{align*}
Using the fact that $h_f$ is a Bernstein class, and noting that for $\beta\in[0,1)$, $z^\beta \leq 1+z$ for all $z>0$,
\begin{align*}
	\frac{\mathbb{E} \{h_f(\bX,Y)^2\}}{t+\mathbb{E} \{h_f(\bX,Y)\}}
		&\leq B^\prime \cdot \frac{\mathbb{E} \{h_f(\bX,Y)\}^\beta}{t+\mathbb{E} \{h_f(\bX,Y)\}}. \\
		&\leq B^\prime \cdot \frac{1+\mathbb{E} \{h_f(\bX,Y)\}}{t+\mathbb{E} \{h_f(\bX,Y)\}} \\
		&\leq B^\prime \cdot t^{-1}.
\end{align*}
Therefore,
\begin{align*}
	\mathbb{P}\Big\{ \sup_{f\in\mathcal{F}_n} \Big( \mathbb{E} \{h_{f}(\bX, Y)\} - 
			2 \big(\frac{1}{n} \sum_{i=1}^n h_{f} (\bx_i, y_i) \big)\Big) \geq t\Big\} 
		&\\
		\leq N_n \cdot \text{exp} &\Big\{ - \frac{n}{8}\cdot
			\frac{(t + \mathbb{E}\{h_f(\bX,Y)\})}%
			{B^\prime + \tfrac{LB}{3}} 
			\Big\}
		&\\
		\leq N_n \cdot \text{exp} &\Big\{ - \frac{nt}{8}\cdot
			\Big(\frac{B^\prime}{t} + \frac{LB}{3}\Big)^{-1} 
			\Big\}.
\end{align*}
The proof is complete by noting that the necessary bound holds with probability $\gamma$ for:
\begin{align*}
	t&= 4 \cdot \frac{LB}{3} \cdot \frac{\log(N_n / \gamma)}{n} + 
		\left( \left(4 \cdot \frac{LB}{3}\cdot \frac{\log(N_n / \gamma)}{n} \right)^2 + 
			8 \cdot B^\prime \cdot \frac{\log(N_n / \gamma)}{n} \right)^{1/2}.
\end{align*}

\end{document}